\documentclass{article}

\usepackage[final]{neurips_2025}

\usepackage[utf8]{inputenc} % allow utf-8 input
\usepackage[T1]{fontenc}    % use 8-bit T1 fonts
\usepackage{hyperref}       % hyperlinks
\usepackage{url}            % simple URL typesetting
\usepackage{booktabs}       % professional-quality tables
\usepackage{amsfonts}       % blackboard math symbols
\usepackage{nicefrac}       % compact symbols for 1/2, etc.
\usepackage{microtype}      % microtypography
\usepackage{xcolor}         % colors
\usepackage{multirow}
\usepackage{amsmath}
\usepackage{amssymb}
\usepackage{mathtools}
\usepackage{amsthm}
\usepackage{algorithm}
\usepackage{algorithmic}
\usepackage[capitalize,noabbrev]{cleveref}
\usepackage{wrapfig}
\theoremstyle{plain}
\newtheorem{theorem}{Theorem}[section]

\newtheorem{lemma}{Lemma}[section]
\newtheorem{corollary}{Corollary}[section]
\theoremstyle{definition}
\newtheorem{definition}{Definition}
\newtheorem{assumption}{Assumption}[section]
\theoremstyle{remark}
\newtheorem{remark}{Remark}[section]

\usepackage{tabularx, booktabs, array}
\newcolumntype{Y}{>{\centering\arraybackslash}X}

\title{Synergy over Discrepancy: A Partition-Based Approach to Multi-Domain LLM Fine-Tuning}

\author{
  Hua Ye\textsuperscript{1,2},
  Siyuan Chen\textsuperscript{3},
  Haoliang Zhang\textsuperscript{4},
  Weihao Luo\textsuperscript{5},
  Yanbin Li\textsuperscript{6},
  Xuan Zhang\textsuperscript{2,7\textdagger}
  \\
  \small
  \textsuperscript{1}Nanjing University \quad
  \textsuperscript{2}Airon Technology CO., LTD \quad
  \textsuperscript{3}University of Bristol \\
  \small
  \textsuperscript{4}The University of Oklahoma \quad
  \textsuperscript{5}Donghua University \\
  \small
  \textsuperscript{6}Beijing University of Posts and Telecommunications \quad
  \textsuperscript{7}Carnegie Mellon University
}

\makeatletter
\let\oldmaketitle\maketitle
\renewcommand{\maketitle}{%
  \oldmaketitle%
  \thispagestyle{plain}%
  \begingroup
  \renewcommand{\thefootnote}{}%
  \footnotetext{\textdagger Corresponding author(xuanzhang2199@gmail.com)}%
  \endgroup
}
\makeatother

\begin{document}

\maketitle

\begin{abstract}
  Large language models (LLMs) demonstrate impressive generalization abilities, yet adapting them effectively across multiple heterogeneous domains remains challenging due to inter-domain interference. To overcome this challenge, we propose a partition-based multi-stage fine-tuning framework designed to exploit inter-domain synergies while minimizing negative transfer. Our approach strategically partitions domains into subsets (stages) by balancing domain discrepancy, synergy, and model capacity constraints. We theoretically analyze the proposed framework and derive novel generalization bounds that justify our partitioning strategy. Extensive empirical evaluations on various language understanding tasks show that our method consistently outperforms state-of-the-art baselines. %Importantly, our approach achieves enhanced accuracy with reduced memory overhead, highlighting the benefits of structured, synergy-aware domain grouping in multi-domain LLM adaptation.
\end{abstract}

\section{Introduction}
\label{sec:intro}
Large language models (LLMs) have propelled natural language processing (NLP) to unprecedented capabilities~\citep{kumar2024large, karanikolas2023large,hu-etal-2025-removal,zhang2025enhancing,lin2025cec}, owing largely to their extensive pretraining on massive, diverse textual corpora~\citep{wu2022autoformalization, li2024driving,chen2025framework,zhang2024cf,10.1145/3711896.3737195}. Fine-tuning these pretrained models for a specific downstream domain has been widely explored and proven highly effective; notable approaches include adapter-based modules~\citep{houlsby2019parameter, zhang2024llama}, parameter-efficient fine-tuning via low-rank updates~\citep{hu2021lora}, and instruction-based fine-tuning methods~\citep{cao2024instruction}. These methods have also accelerated the application of LLMs in many scenarios, such as healthcare\citep{tong2025renaissance,jiaqi2025lightweight,wang2025medical}, transportation\citep{yao2023ndc,lu20254d,zeng2025janusvln,zeng2025FSDrive}, robotics\citep{xiao2025diffusion,yan2025hemora,zhang2025yoloppa}, and other applications~\citep{sun2025objective,zhang2025omniguard}.

However, while these methods excel in adapting to a single domain, practical scenarios frequently require simultaneous adaptation to multiple distinct domains-a scenario far less studied and substantially more challenging~\citep{lu2025uniugp,DePro}. 
Due to constraints on energy or computing power~\citep{He2025UnifiedMetric}, we must consider how to enhance the model's ability to adapt to different domains.
Consider, for instance, a scenario where a single pretrained model must be simultaneously adapted to clinical texts~\citep{thirunavukarasu2023large}, social media posts~\citep{yang2024mentallama,10.1145/3664647.3681115,jiang2025transforming}, and legal documents~\citep{seabra2024contrato360}. Naive approaches such as jointly fine-tuning the model across all domains or independently fine-tuning separate models per domain often yield suboptimal results\citep{xiao2025curiosity,zhang2023multi,wang2024computing,wen2023syreanet}: domain-specific features may negatively interfere, causing one domain to overshadow others or impair overall generalization capabilities~\citep{lu2024fine, zheng2024fine, van2023radadapt,tan2025profix}. This motivates our central research question: \emph{how can we effectively and efficiently fine-tune a single LLM across multiple heterogeneous domains, exploiting inter-domain synergies while mitigating negative interference?}

%Large language models (LLMs) have propelled natural language processing (NLP) to unprecedented capabilities~\citep{kumar2024large,karanikolas2023large}. Thanks to massive pretraining on large-scale corpora, these models often exhibit strong generalization in the wild~\citep{wu2022autoformalization,li2024driving}. Nonetheless, \emph{multi-domain} adaptation remains a persistent challenge: practitioners frequently need to fine-tune an LLM on datasets originating from distinct distributions (e.g., clinical texts~\citep{thirunavukarasu2023large}, social media posts~\citep{yang2024mentallama}, news articles~\citep{papageorgiou2024survey}, Q\&A repositories~\citep{seabra2024contrato360}). Simple approaches-like fine-tuning every domain jointly under a single objective or treating each domain in isolation-can lead to suboptimal outcomes: one domain may overshadow another, or domain-specific signals can ``contaminate'' the learned parameters and reduce overall generality~\citep{lu2024fine,zheng2024fine,van2023radadapt}. As LLMs grow in size, memory and computational constraints further complicate naive fine-tuning strategies~\citep{guo2022domain,wang2022exploring}.

A natural strategy to mitigate these challenges is to add adapter modules specialized for each domain~\citep{houlsby2019parameter,zhang2024llama,tao2023dudb}. While adapters reduce the need for full fine-tuning, they do not fully address the complexity of managing multiple source domains with potentially large discrepancies. In such cases, the learned model has to juggle both shared features (common linguistic properties across domains) and domain-specific features (rare words, styles, or content)~\citep{lu2022understanding}. Existing approaches often handle these competing demands by imposing either domain adversarial objectives~\citep{ganin2015unsupervised}, distribution alignment~\citep{peng2019moment}, or low-rank parameter updates~\citep{hu2021lora}. However, these techniques may still fail to exploit synergistic relationships-where certain domains are complementary and can reinforce each other's accuracy-and do not fully capture how best to ``partition'' domains to avoid negative interference.

Motivated by this gap, we propose a partition-based multi-stage framework for multi-domain LLM fine-tuning. Instead of jointly or separately adapting to all domains, our method clusters synergistic domains while isolating highly distinct ones. Our approach integrates practical constraints-memory budgets, domain shifts, and synergy opportunities-with theoretical insights on generalization benefits from restricted updates and strategic domain grouping. The main contributions are:

1) We introduce a novel partitioning algorithm that clusters domains according to their synergy and discrepancy, then fine-tunes the LLM in multiple stages. This orchestrated process prevents cross-domain contamination while leveraging beneficial interactions.

2) We derive new bounds that capture domain discrepancy, synergy offsets, and adapter complexity, establishing conditions under which multi-stage partitioning yields provably tighter guarantees than single-stage or naive multi-domain methods.

3) Our experiments demonstrate that partition-based multi-stage fine-tuning outperforms state-of-the-art baselines, improving accuracy across different domains and tasks while reducing memory usage.

\section{Related Works}
\label{sec:related}

\subsection{LLM Fine-Tuning}
The success of LLMs such as GPT-3~\citep{brown2020language}, LLaMA~\citep{touvron2023llama}, and Falcon~\citep{almazrouei2023falcon} has underscored the importance of \emph{efficient} fine-tuning. Fully updating model parameters incurs high computational costs~\citep{zhang2024llama,hu2021lora}, prompting parameter-efficient methods that modify only a small subset of parameters. Examples include adapter modules~\citep{houlsby2019parameter}, low-rank projections (LoRA)~\citep{hu2021lora}, and tightly integrated adapters for LLaMA~\citep{zhang2024llama}. However, most methods assume a single domain; adapting LLMs efficiently to multiple domains remains challenging. Another related field is continual learning~\citep{xu2025affordance}, but it assumes a sequential arrival of tasks, which differs from the setting in this paper.

\subsection{Multi-Domain Data Learning}
Real-world data often originate from disparate sources with distinct distributions and vocabularies~\citep{ganin2015unsupervised,sener2018multi,li2023hong,zhang2024m3oe,xiao2024confusion,xiao2025multifrequency,ke2025early}. Traditional multi-domain methods align representations through adversarial training~\citep{pei2018multiadversarial,ganin2015unsupervised}, moment matching~\citep{peng2019moment}, or multi-task objectives~\citep{royer2024scalarization,shen2025aienhanced}, yet typically neglect inter-domain synergies. Adapting LLMs further complicates this via memory overhead, forgetting pretrained knowledge, and cross-domain contamination. Adapter-based solutions partially address these concerns~\citep{houlsby2019parameter} but rarely exploit domain partitioning to maximize synergy. Our partition-based multi-stage approach systematically clusters domains to leverage synergy and minimize discrepancy, providing theoretical guarantees.

\subsection{LLM Data Selection}
The efficacy of supervised fine-tuning (SFT) heavily depends on the quality and composition of training data\citep{yao2024swift}. Recent advances have introduced diverse metrics for data selection: \emph{instance-level} criteria like perplexity~\citep{cao2024instruction}, reward scores~\citep{gou2024mixed}, and loss disparities~\citep{li2024quantity}, 
trajectory-based clustering via small proxy models~\citep{yang2024smalltolarge}, as well as token-level selection methods~\citep{lin2024not}.
Existing works largely overlook one critical factor: \emph{domain interactions}, where diversity metrics operate at instance/token levels~\citep{pang2024improving} without modeling cross-domain compatibility~\citep{li2025ammkd}. Our approach addresses this gap by systematically clustering domains for joint tuning.

\section{Theoretical Analysis}
\label{sec:theoretical-analysis}

\subsection{Preliminaries and Notation}

Let us consider $k$ distinct source domains $\{\mathcal{D}_1,\dots,\mathcal{D}_k\}$, each containing samples $(x,y)$ drawn from some distribution over $\mathcal{X}\times\mathcal{Y}$. We have a large language model $\mathrm{LLM}_{\theta^*}$ pretrained on a massive corpus $\mathcal{D}_{\mathrm{pretrain}}$, and we assume $\theta^*\in\mathbb{R}^p$ lies in a high-dimensional parameter space. Our goal is to \emph{adapt} $\theta^*$ (often minimally) to each domain $\mathcal{D}_j$ by introducing or modifying a small set of parameters (e.g., adapter modules) denoted by $\phi_j \in \mathbb{R}^{q_j}$, where $q_j \ll p$.

\begin{definition}[Multi-Source Fine-Tuned Model]
\label{def:multi-domain-model}
A \textbf{multi-source fine-tuned model} is given by
\begin{equation}
\label{eq:f-def}
    f_{\theta,\{\phi_j\}}(x) \;=\; \mathrm{LLM}_{\theta^*+\Delta\theta}\Bigl(\phi_1,\dots,\phi_k\Bigr)(x),
\end{equation}
where $\Delta\theta \in \mathbb{R}^p$ is a (potentially small) update to the pretrained backbone $\theta^*$. Each $\phi_j$ may represent additional parameters specialized to domain $j$. The model internally selects or combines relevant $\phi_j$ based on domain context or training strategy.
\end{definition}

\paragraph{Loss and Risk.}  
We let $\ell\bigl(f(x),y\bigr)$ be a nonnegative loss function (e.g., cross-entropy) measuring the prediction error on a sample $(x,y)$. For a single domain $\mathcal{D}_j$, the \emph{expected risk} is
\begin{equation}
\mathcal{L}(\theta,\{\phi_j\};\mathcal{D}_j)
\;=\;
\mathbb{E}_{(x,y)\sim \mathcal{D}_j}
\Bigl[\ell\bigl(f_{\theta,\{\phi_i\}}(x),\,y\bigr)\Bigr].
\end{equation}
When training on multiple domains jointly, we typically minimize an aggregated objective:
\begin{equation}
\mathcal{L}_{\mathrm{agg}}(\theta,\{\phi_j\})
\;=\;
\sum_{j=1}^k \alpha_j\;\mathcal{L}\bigl(\theta,\{\phi_i\};\mathcal{D}_j\bigr),
\end{equation}
where $\alpha_j \ge 0$ with $\sum_j \alpha_j=1$. If $\alpha_j=\frac{1}{k}$ for all $j$, we obtain a simple average risk.

\subsection{Assumptions and Domain Discrepancies}

To handle multi-source adaptation rigorously, we introduce assumptions on data distributions, smoothness, and domain overlaps.

\begin{assumption}[Lipschitz Loss and Smoothness]
\label{assump:lipschitz}
Assume $\ell(\hat{y},y)$ is $L$-Lipschitz in $\hat{y}$. Moreover, suppose for any $(\theta,\{\phi_j\})$ and $(\theta',\{\phi'_j\})$, the difference in model outputs is bounded by a constant factor in terms of $\|\theta-\theta'\|_2$ and $\|\phi_j-\phi'_j\|_2$. Formally, there exists a constant $B>0$ such that
\begin{equation}
\bigl\|f_{\theta,\{\phi_j\}}(x)-f_{\theta',\{\phi'_j\}}(x)\bigr\|\;\le\;
B\Bigl(\|\theta-\theta'\|_2 + \sum_{j=1}^k\|\phi_j-\phi'_j\|_2\Bigr)\,.
\end{equation}
\end{assumption}

\begin{definition}[Domain Discrepancy]
\label{def:domain-discrepancy}
Let $d(\mathcal{D}_i,\mathcal{D}_j)$ be the \emph{$\mathcal{H}\Delta\mathcal{H}$-distance} between two domain distributions $\mathcal{D}_i$ and $\mathcal{D}_j$. Concretely,
\begin{equation}
d(\mathcal{D}_i,\mathcal{D}_j)
\;=\;
\sup_{h\in\mathcal{H}}
\;\Bigl|\,
\Pr_{x\sim \mathcal{D}_i}\bigl[h(x)=1\bigr]
\;-\;
\Pr_{x\sim \mathcal{D}_j}\bigl[h(x)=1\bigr]
\,\Bigr|,
\end{equation}
where $\mathcal{H}$ is a suitable hypothesis class.
\end{definition}

\subsection{Complexity of Multi-Source Adapter Updates}

To preserve the implicit regularization from pretraining (i.e., the beneficial ``low-complexity'' region $\theta^*$ has converged to), one typically \emph{restricts} either $\Delta\theta$ or $\{\phi_j\}$ or both. We quantify this through norms/penalties:

\begin{assumption}[Restricted Adapter Complexity]
\label{assump:adapter-capacity}
There exist constants $\rho_\theta,\rho_\phi>0$ such that
\begin{equation}
\|\Delta\theta\|_2 \;\le\; \rho_\theta,
\quad
\|\phi_j\|_2 \;\le\; \rho_\phi \quad \forall\,j\in\{1,\dots,k\}.
\end{equation}
If $\rho_\theta$ is very small (or zero), this means the backbone remains near $\theta^*$; if $\rho_\phi$ is small, each domain adapter is limited in capacity.
\end{assumption}

Consider a Transformer-based large language model (LLM) with $L$ layers, each layer containing a multi-head attention (MHA) sub-layer and a feed-forward network (FFN) sub-layer, plus optional adapter modules for each of $k$ domains. Suppose each attention weight matrix $W_{\text{attn}}^{(\ell)}$ and feed-forward matrix $W_{\text{ffn}}^{(\ell)}$ is constrained by a spectral norm bound (or operator norm) $\|W\|_{\sigma} \le \Omega_{\text{core}}$. Each domain adapter $\phi_j^{(\ell)}$ at layer $\ell$ is constrained by $\|\phi_j^{(\ell)}\|_F \le \Omega_{\text{adapt}}$. We assume a bounded input embedding norm $\|x\|\le C_{\text{in}}$ for sequences of finite length $m$, the nonlinear activations (e.g., GELU, ReLU) are 1-Lipschitz on the relevant domain of outputs.
Under these constraints, we can derive uniform convergence or PAC-Bayes-style bounds. The following lemma refines standard results to the \emph{multi-domain setting} with partial or structured updates.

\begin{lemma}[Rademacher Complexity for Multi-Adapter Transformers]
\label{lemma:rademacher-transformer}
Let $\mathcal{F}$ be the hypothesis class of all such \emph{multi-adapter Transformers} that respect these norm constraints. Then for $n$ i.i.d.\ samples \emph{per domain} from $k$ source domains $\{\mathcal{D}_1,\dots,\mathcal{D}_k\}$, there exists a constant $C_{\text{T}} > 0$ (depending on $L,\,\Omega_{\text{core}},\,\Omega_{\text{adapt}},\,m,\,k,\,C_{\text{in}}$) such that the empirical Rademacher complexity satisfies
\begin{equation}
\widehat{\mathcal{R}}_{n}(\mathcal{F}; \{\mathcal{D}_j\}_{j=1}^k)
\;\le\;
C_{\text{T}} \,\sqrt{\frac{1}{n}},
\end{equation}
indicating that limiting both the core Transformer parameters and the adapter parameters yields a class $\mathcal{F}$ whose complexity grows on the order of $1/\sqrt{n}$.
\end{lemma}

\begin{proof}
Because $\|W_{\text{attn}}^{(\ell)}\|_\sigma,\|W_{\text{ffn}}^{(\ell)}\|_\sigma\le\Omega_{\text{core}}$ and activations are 1-Lipschitz, each layer $\ell$ can be shown to be $(C\cdot\Omega_{\text{core}}^2)$-Lipschitz for some constant $C$.  
\begin{equation}
L_{\ell}\;\le\; C\,(\Omega_{\text{core}}^2)\,.
\end{equation}

Each adapter matrix $\phi_j^{(\ell)}$ satisfies $\|\phi_j^{(\ell)}\|_F \le \Omega_{\text{adapt}}$.  
Hence, adapter operations contribute a bounded perturbation at each layer, maintaining overall Lipschitz continuity.  
The total Lipschitz constant satisfies 
\begin{equation}
L_{\text{total}} \;=\;
\prod_{\ell=1}^L L_{\ell} 
\;\le\;
\bigl(C\,\Omega_{\text{core}}^2\bigr)^L,
\end{equation}
and inputs are bounded by $C_{\mathrm{in}}$.  

By standard covering-number or PAC-Bayes arguments for neural networks \citep{bartlett2002rademacher,neyshabur2017exploring}, any class of $L_{\text{total}}$-Lipschitz functions on inputs of norm at most $C_{\mathrm{in}}$ has empirical Rademacher complexity $O(L_{\text{total}}\,C_{\mathrm{in}}/\sqrt{n})$.  
Absorbing constants (including $k$ for multi-domain) into $C_{\mathrm{T}}$ yields
$
\widehat{\mathcal{R}}_{n}(\mathcal{F};\{\mathcal{D}_j\})
\;\le\;
C_{\mathrm{T}}\,
\sqrt{\tfrac{1}{n}}.
$
\end{proof}

\subsection{Multi-Source Generalization Bounds}

\begin{theorem}[Multi-Source Concurrent Generalization]
\label{thm:multi-domain-gen}
Let $\{\mathcal{D}_1,\dots,\mathcal{D}_k\}$ be $k$ source domains, each with $n_j$ i.i.d.\ samples, and let $n=\sum_{j=1}^k n_j$. 
Assume each domain distribution $\mathcal{D}_j$ is over $(x,y)\in \mathcal{X}\times \mathcal{Y}$, 
and consider a \emph{multi-domain LLM} $f_{\theta,\{\phi_j\}}$ satisfying Assumptions~\ref{assump:lipschitz} (Lipschitzness) and \ref{assump:adapter-capacity} (bounded backbone and adapters).
Let $d(\mathcal{D}_i,\mathcal{D}_j)$ be a domain-discrepancy measure (Definition~\ref{def:domain-discrepancy}), 
and let $\mathcal{F}$ denote the hypothesis class of all $(\theta,\{\phi_j\})$ that respect these constraints.

Then for any confidence level $\delta>0$, with probability at least $1-\delta$ over the choice of 
$\{(x_i,y_i)\}_{i=1}^n$ from $\bigcup_{j=1}^k \mathcal{D}_j$, 
\emph{every} model $f_{\theta,\{\phi_j\}}$ in $\mathcal{F}$ satisfies:
\begin{equation}
\begin{split}
\label{eq:multi-domain-bound-append}
\sum_{j=1}^k \alpha_j \,\mathcal{L}\bigl(\theta,\{\phi_i\};\mathcal{D}_j\bigr) &\le \sum_{j=1}^k \alpha_j\;\widehat{\mathcal{L}}\bigl(\theta,\{\phi_i\};\mathcal{D}_j\bigr) \\
&\quad + \Gamma\bigl(\rho_\theta,\rho_\phi,\{\alpha_j\},k\bigr) + \frac{\beta}{k}\,\sum_{i,j=1}^k d\bigl(\mathcal{D}_i,\mathcal{D}_j\bigr) + O\Bigl(\sqrt{\tfrac{\ln(1/\delta)}{n}}\Bigr),
\end{split}
\end{equation}
where $\widehat{\mathcal{L}}(\theta,\{\phi_i\};\mathcal{D}_j)$ is the \emph{empirical risk} on samples from domain $j$. 
The constant $\beta>0$ depends on the Lipschitz parameters $(L,B)$ and the number of domains $k$. 
The explicit function $\Gamma\bigl(\rho_\theta,\rho_\phi,\{\alpha_j\},k\bigr)$ can be chosen as 
\begin{equation}
\label{eq:Gamma-explicit}
\Gamma\bigl(\rho_\theta,\rho_\phi,\{\alpha_j\},k\bigr)
\;=\;
2\,L\,B
\left(\,
\rho_\theta
\;+\;
\sum_{j=1}^k \alpha_j\,\rho_\phi
\right),
\end{equation}
reflecting how large backbone updates $(\rho_\theta)$ and adapter norms $(\rho_\phi)$ can inflate the multi-domain generalization bound.
\end{theorem}

\begin{proof}[Proof Sketch]
The bound is the sum of three classic ingredients.  
(i) \emph{Uniform-convergence:} using Rademacher complexity for the norm-restricted class $\mathcal{F}$, the difference between the weighted expected risk $\sum_j\alpha_j\mathcal{L}$ and its empirical counterpart is $O\!\bigl(LB(\rho_\theta+\sum_j\alpha_j\rho_\phi)+\sqrt{\tfrac{\ln(1/\delta)}{n}}\bigr)$, giving the term $\Gamma(\rho_\theta,\rho_\phi,\{\alpha_j\},k)=2LB(\rho_\theta+\sum_j\alpha_j\rho_\phi)$.  
(ii) \emph{Domain-shift:} standard multi-source adaptation results add a penalty proportional to the average pairwise discrepancy $\tfrac{\beta}{k}\sum_{i,j}d(\mathcal{D}_i,\mathcal{D}_j)$.  
(iii) Combining these with the empirical risk yields inequality \eqref{eq:multi-domain-bound-append}. Refer to Section \ref{appendix:thm-concurrent-full-proof} for the complete proof.
\end{proof}

\begin{remark}[Domain Similarity {\small vs.} Model Capacity]
Let  
\(
D_{\max}:=\max_{i,j} d(\mathcal{D}_{i},\mathcal{D}_{j})
\)
and recall that the discrepancy penalty in
Theorem~\ref{thm:multi-domain-gen} is  
\(\tfrac{\beta}{k}\sum_{i,j}d(\mathcal{D}_{i},\mathcal{D}_{j})
\;\le\;\beta D_{\max}\).
Hence, when all domains are \emph{similar} ($D_{\max}\!\ll\!1$) the
extra cost is small and the bound is dominated by the complexity term
$\Gamma(\rho_\theta,\rho_\phi,\boldsymbol{\alpha},k)
   =2LB(\rho_\theta+\sum_j\alpha_j\rho_\phi)$.
This means one can keep \(\rho_\theta,\rho_\phi\)---and thus
\(\Gamma\)---\emph{small} without under-fitting.
Conversely, if the domains are very different ($D_{\max}$ large) the
discrepancy term becomes the bottleneck; the learner must allow a
\emph{larger} parameter budget (bigger
$\rho_\theta,\rho_\phi\!\Rightarrow\!\Gamma$) so that each domain
receives enough specialised capacity to avoid under-fitting.
\end{remark}

To trade off \emph{discrepancy}, \emph{synergy}, and per-stage
\emph{capacity}, we partition the $k$ domains into
$M$ disjoint stages $S_1,\dots,S_M$ and maximise

\begin{align}
\label{eq:G-def}
\mathcal{G}\!\bigl(S_1,\dots,S_M\bigr)
&=\!
-\sum_{t=1}^{M}\!
\Bigl[
      \underbrace{\sum_{\substack{i,j\in S_t\\ i<j}} 
      d(\mathcal{D}_{i},\mathcal{D}_{j})}_{\text{total\,discrepancy}}
      \;-\;
      \lambda\!
      \underbrace{\sum_{\substack{i,j\in S_t\\ i<j}}
      s(\mathcal{D}_{i},\mathcal{D}_{j})}_{\text{total\,synergy}}
      \;+\;
      \underbrace{\mu_\theta\|\Delta\theta^{t}\|_{2}^{2}
      +\mu_\phi\!\!\sum_{j\in S_t}\!\|\phi^{t}_{j}\|_{2}^{2}}_{\text{capacity\,cost } \mathrm{Cap}(S_t)}
\Bigr],
\end{align}

\vspace{-1ex}
where \(
d(\mathcal{D}_{i},\mathcal{D}_{j})
   :=\mathrm{JS}\!\bigl(P_{i},P_{j}\bigr)
   \in[0,1]
\)
is the Jensen-Shannon divergence between the empirical
token-distribution of the two domains;
\(
s(\mathcal{D}_{i},\mathcal{D}_{j})
   :=\tfrac12\Bigl(
        \underbrace{\operatorname{Jacc}\!(V_{i},V_{j})}_{\text{vocab-overlap}}
        +
        \underbrace{\cos(\mu_{i},\mu_{j})}_{\text{mean-embedding\,cosine}}
     \Bigr)
   \in[0,1]
\)
combines lexical and semantic affinity (higher\,=\,more synergy);
$\lambda>0$ balances ``rewarding'' synergy against ``penalising''
      discrepancy;
$\mu_\theta,\mu_\phi>0$ weight the squared-norm budget of the
      backbone drift
      $\Delta\theta^{t}:=\theta^{t}-\theta^{t-1}$
      and the stage-specific adapters $\{\phi^{t}_{j}\}$.
     
A larger value of $\mathcal{G}$ therefore corresponds to:
\emph{(i)} smaller internal discrepancies,
\emph{(ii)} larger constructive synergy,
and \emph{(iii)} lower per-stage parameter cost.
Maximising~\eqref{eq:G-def} over all $M$-partitions yields the
partition that minimises the generalisation upper-bound derived in
Theorem~\ref{thm:advanced-partition}.

\begin{theorem}[Multi-Stage Partition with Synergy-Capacity Maximisation]
\label{thm:advanced-partition}
Let the $k$ source domains be split into $M$ \emph{disjoint} stages
$S_1,\ldots,S_M$ and let the stage-objective
$\mathcal{G}(S_1,\ldots,S_M)$ be defined in~\eqref{eq:G-def}.
Write
\(
(S_1^{*},\ldots,S_M^{*})
:=\arg\max_{\,\bigsqcup_{t}S_t=\{1{:}k\}}\mathcal{G}(S_1,\ldots,S_M)
\).
Then, under Assumptions~\ref{assump:lipschitz} 
 and \ref{assump:adapter-capacity},
the predictor obtained after the \emph{last} stage,
$f_{\theta^{M},\{\phi_j^{M}\}}$,
satisfies with probability at least $1-\delta$:
\begin{equation}
\label{eq:multi-stage-bound}
\;
\mathcal{R}_{\max}\!\bigl(S_1^{*},\ldots,S_M^{*}\bigr)
\;\le\;
\bigl[\,1-\mathcal{G}(S_1^{*},\ldots,S_M^{*})\bigr]_{+}
\;+\;
O\!\bigl(\sqrt{\tfrac{\ln(1/\delta)}{N}}\bigr)
\;
\end{equation}
where $N=\sum_{j=1}^{k}n_j$, 
$\mathcal{R}_{\max}:=\max_{t}\sum_{j\in S_t}\alpha_{j}^{t}
                      \,\mathcal{L}_{\mathcal{D}_j}\bigl(f\bigr)$,
and $[u]_{+}:=\max\{0,u\}$.
Any other partition attains a \textit{larger} right-hand side.
\end{theorem}

\begin{proof}[Proof Sketch]
Apply the single-stage bound (Theorem~\ref{thm:multi-domain-gen})
stage-wise.  
For stage~$t$ the risk is controlled by empirical loss
$+\,$capacity$+\,$discrepancy$-\lambda\,\text{synergy}$.
Summing the worst stage and noting that empirical losses are
$\le1$ yields~\eqref{eq:multi-stage-bound}.
Because $-\mathcal{G}(\cdot)$ appears inside the bracket, maximising
$\mathcal{G}$ minimises the bound, proving optimality of the
partition $(S_1^{*},\ldots,S_M^{*})$.
A full derivation is given in
Appendix~\ref{app:partition-proof}.
\end{proof}

\begin{corollary}[High-Synergy Subset Tends to be Grouped Together]
\label{cor:high-synergy-grouping}
Let $\{\mathcal{D}_1,\dots,\mathcal{D}_k\}$ be $k$ domains with a synergy-discrepancy-capacity objective $\mathcal{G}(\{S_t\})$ as defined in \eqref{eq:G-def}. Suppose there exists a nonempty subset $U \subseteq \{1,\dots,k\}$ such that any pair $(i,j)$ in $U$ satisfies
\begin{equation}
d\bigl(\mathcal{D}_i,\mathcal{D}_j\bigr) 
\;\le\; \gamma 
\quad\text{and}\quad
\mathrm{Synergy}\bigl(\mathcal{D}_i,\mathcal{D}_j\bigr) 
\;\ge\; \Lambda,
\end{equation}
where $\Lambda$ is large relative to $\gamma$ and to the capacity penalty $\mathrm{Cap}(U)$. Then, in the optimal partition 
$
\arg\max_{\,\{S_1,\dots,S_M\}} \;\mathcal{G}\bigl(\{S_1,\dots,S_M\}\bigr),
$
the domains in $U$ will typically be placed in a single stage $S_t^*$, provided
\begin{equation}
\Lambda \;>\; \lambda^{-1}\bigl(\gamma + \mathrm{Cap}(U)\bigr).
\end{equation}
That is, if the synergy within $U$ is sufficiently large compared to its internal discrepancy and added capacity cost, then clustering those domains together in the same stage yields a higher objective $\mathcal{G}$, thereby tightening the final multi-stage generalization bound.
\end{corollary}

\begin{proof}
Assume, for contradiction, that $U$ is split across multiple stages in the supposed optimal partition. 
Because synergy offsets discrepancy by $\lambda\,\mathrm{Synergy}(\cdot,\cdot)$,  
each pair $(i,j)\in U$ that lies in different stages forfeits this positive synergy benefit.  
Thus, the total contribution to $\mathcal{G}(\cdot)$ from $U$ decreases by at least 
$\lambda(\Lambda - \tfrac{\gamma}{\lambda})$ per cross-stage pair, which outweighs any savings in capacity usage provided that 
$\Lambda > \frac{\gamma + \mathrm{Cap}(U)}{\lambda}$.  
Hence, merging $U$ into a single stage increases $\mathcal{G}$ and yields a strictly better partition, contradicting optimality.  
Thus $U$ must remain in one stage in $\{S_1^*,\dots,S_M^*\}$. 
\end{proof}

\section{Algorithm}
\label{sec:algorithm}

We now present a practical procedure implementing our theoretical insights from Section~\ref{sec:theoretical-analysis}.  
Algorithm~\ref{alg:ms-partition-adapter} details the steps to: 
\emph{(1) partition} $k$ domains into up to $M$ stages (sets) to maximize synergy and control discrepancy/capacity, 
and 
\emph{(2) perform stage-wise adapter tuning} under bounding norms for both the LLM backbone and the domain-specific adapters.

\begin{algorithm}[h]
\caption{Multi-Stage Adapter Tuning for LLMs}
\label{alg:ms-partition-adapter}
\begin{algorithmic}[1]

\REQUIRE 
Pretrained LLM parameters $\theta^*\in\mathbb{R}^p$; 
$k$ source domains $\{\mathcal{D}_1, \dots, \mathcal{D}_k\}$; 
discrepancy measure $d(\mathcal{D}_i,\mathcal{D}_j)$; 
synergy measure $\mathrm{Synergy}(\mathcal{D}_i,\mathcal{D}_j)$; 
capacity cost $\mathrm{Cap}(\cdot)$; 
norm bounds $\rho_\theta,\rho_\phi$; 
number of stages $M$; 
(optional) mixing weights $\{\alpha_j^t\}$.  

\ENSURE 
Final backbone parameters $\theta^M$; 
domain adapter parameters $\{\phi_j^M\}_{j=1}^k$.

\vspace{0.3em}
\STATE \textbf{Partition step:} 
Select disjoint subsets $\{S_1,\dots,S_M\}$ of $\{1,\dots,k\}$ to approximately solve the objective given in the theoretical section (see \eqref{eq:G-def} and Theorem~\ref{thm:advanced-partition}).

\vspace{0.3em}
\STATE \textbf{Initialize:} 
$\theta^0 \gets \theta^*$, 
$\phi_j^0 \gets \mathbf{0}$ for $j=1,\dots,k$.

\vspace{0.3em}
\FOR{$t = 1$ to $M$}
  \STATE \textbf{Stage-$t$ domains:} $S_t$ determined by the partition in Line 1.
  \STATE \textbf{Form the stage objective:} Use the multi-domain loss from Equation \eqref{eq:multi-domain-bound-append}, enforcing $\|\theta^t - \theta^{t-1}\|_2 \le \rho_\theta$ and $\|\phi_j^t\|_2 \le \rho_\phi$.
  \STATE \textbf{Optimize:} 
  \[
    (\theta^t,\{\phi_j^t\}_{j \in S_t}) 
    \;\gets\; 
    \mathrm{Optimizer}\bigl(\theta^{t-1}, \{\phi_j^{t-1}\}, \mathcal{D}_{S_t}\bigr).
  \]
  \STATE \textbf{Outside-stage adapters:} 
  \[
    \phi_j^t = \phi_j^{t-1} \quad \text{for}\; j\notin S_t.
  \]
\ENDFOR

\vspace{0.3em}
\STATE \textbf{Output:} 
$
  \theta^M,\quad \{\phi_j^M\}_{j=1}^k.
$
\end{algorithmic}
\end{algorithm}

\paragraph{Computational complexity.}
The only extra overhead of our method occurs during the partition step.  
Forming the discrepancy and synergy matrices requires \(O(k^{2})\) pairwise
computations, each obtained once from cached token or embedding statistics.
We maximise \(\mathcal{G}\) with a single-link agglomerative search,
which runs in \(O(k^{2}\!\log k)\) time and \(O(k^{2})\) memory; an exact ILP
solver gives the same split for our \(k\!\le\!10\) domains in under
\(0.1\) s, but the heuristic is already within \(1\%\) of the optimum.
Afterwards, each stage performs \emph{supervised fine-tuning}
(SFT) of the LLM on its assigned data once-no replay or re-weighting-so
runtime and GPU memory are identical to a standard single-pass SFT run,
apart from the tiny adapter parameters (\(<1\%\) of the backbone).  
Overall complexity is therefore \(O(k^{2}\!\log k)+\text{(single-pass SFT)}\);
with the moderate domain counts typical in practice, the partition phase
is negligible in both time and memory.

%\paragraph{Convergence and Generalization (Theorem~\ref{thm:multi-domain-gen}).}
%By sequentially applying these stage-wise updates and respecting the partition that maximizes $\mathcal{G}(\cdot)$, we can leverage the multi-source generalization bound in Theorem~\ref{thm:multi-domain-gen} (and its multi-stage extension in Theorem~\ref{thm:advanced-partition}).  In particular, restricting $\|\Delta\theta\|$ and $\|\phi_j\|$ via Assumption~\ref{assump:adapter-capacity} ensures a low-capacity hypothesis class, thus maintaining small Rademacher complexity (Lemma~\ref{lemma:rademacher-transformer}).  
%As a result, each domain's expected risk converges at a rate on the order of $O\bigl(1/\sqrt{n}\bigr)$ (where $n$ is the domain's sample size), and the aggregated multi-domain performance is further improved by the partition-based synergy.  In short, controlling both (i) \emph{how} we cluster domains (to exploit synergy) and (ii) \emph{how large} our updates can be (to maintain low complexity) yields provable convergence guarantees and favorable multi-domain generalization bounds.

\section{Experiments}
\paragraph{Datasets}
\label{subsec:datasets}
We evaluate our method on four representative multi-domain language understanding tasks: 1) News Summarization (NSum)~\cite{hermann2015teaching}. A dataset of news articles paired with short summaries. We measure summarization quality via ROUGE-L. 2) Sentiment Classification (Sent)~\cite{socher2013recursive}. Sentences labeled with positive/negative sentiment. We measure accuracy (ACC). 
3) Question Answering (Q\&A)~\cite{rajpurkar2016squad}. Documents and question-answer pairs. We measure exact-match (EM) and F1 scores. 
4) Topic Categorization (Topic)~\cite{zhang2015character}. Short text passages assigned to 5~coarse-grained topics. We measure classification accuracy (ACC).
We partition each dataset into training, validation, and test splits. Statistics (number of samples, average text length, etc.) are presented in Appendix~\ref{app:data}.  

\paragraph{Pretrained Models}
\label{subsec:pretrained-models}
We employ three popular open-source large language models (LLMs), all of which are publicly available via the HuggingFace Transformers library:
1) LLaMA2-7B \cite{touvron2023llama}: A 7-billion-parameter model trained on a large, diverse corpus. 
2) LLaMA2-13B\cite{touvron2023llama}. A 13-billion-parameter model offering improved capacity and performance over the 7B variant. 3) Falcon-40B\cite{almazrouei2023falcon}. A 40-billion-parameter model pretrained on the RefinedWeb dataset, demonstrating state-of-the-art generative abilities.
Each model is pretrained on diverse textual sources. We use their publicly released checkpoints for all experiments.%, ensuring a consistent initialization and facilitating reproducibility.

\begin{table*}[th]
\centering
\caption{Performance comparison on three LLM backbones. PMS-FTP denotes our proposed \emph{Partition-Based Multi-Stage Fine-Tuning}. Best results are in \textbf{bold}.}
\label{tab:results-main}
\resizebox{\linewidth}{!}{
\begin{tabular}{lcccc|cccc|cccc}
\toprule
& \multicolumn{4}{c|}{\textbf{LLaMA2-7B}} & \multicolumn{4}{c|}{\textbf{LLaMA2-13B}} & \multicolumn{4}{c}{\textbf{Falcon-40B}} \\
\textbf{Method} & NSum & Q\&A & Sent & Topic & NSum & Q\&A & Sent & Topic & NSum & Q\&A & Sent & Topic \\ 
\midrule
\textit{Base Methods} & & & & & & & & & & & & \\
FULL                   & 41.2 & 64.7 & 89.0 & 86.5 & 42.1 & 66.3 & 89.8 & 87.1 & 43.2 & 68.2 & 90.4 & 88.3 \\
FIXED                  & 38.9 & 59.5 & 87.4 & 85.2 & 39.6 & 61.2 & 88.3 & 85.7 & 40.7 & 63.0 & 88.9 & 86.1 \\
\midrule
\textit{Domain Adaptation} & & & & & & & & & & & & \\
MDAN~\citep{pei2018multiadversarial}      & 39.7 & 62.8 & 88.1 & 85.9 & 40.5 & 64.0 & 88.9 & 86.3 & 41.7 & 66.1 & 89.3 & 87.0 \\
M$^3$SDA~\citep{peng2019moment}           & 40.5 & 63.1 & 88.6 & 86.1 & 41.7 & 64.9 & 89.4 & 86.7 & 42.3 & 66.6 & 89.9 & 87.4 \\
GMDI~\citep{ling2024bayesian}             & 40.8 & 63.5 & 88.7 & 86.4 & 42.0 & 65.4 & 89.6 & 87.0 & 42.7 & 67.1 & 90.0 & 87.6 \\
\midrule
\textit{Single-Domain LLM Fine-tuning} & & & & & & & & & & & & \\
LoRA~\citep{hu2021lora}                   & 41.0 & 63.9 & 88.4 & 86.2 & 42.0 & 65.1 & 89.1 & 86.9 & 42.5 & 66.5 & 89.8 & 87.7 \\
Adapter~\citep{houlsby2019parameter}      & 41.3 & 64.1 & 88.9 & 86.3 & 42.3 & 65.7 & 89.5 & 87.2 & 42.9 & 67.0 & 90.2 & 88.0 \\
LLaMA-Adapter~\citep{zhang2024llama}      & 41.5 & 64.3 & 89.1 & 86.7 & 42.6 & 65.9 & 89.7 & 87.5 & 43.1 & 67.3 & 90.3 & 88.1 \\
Q-LoRA~\citep{dettmers2023qlora}          & 41.7 & 64.4 & 89.0 & 86.5 & 42.4 & 65.6 & 89.5 & 87.3 & 43.0 & 67.2 & 90.1 & 87.9 \\
Tag-LLM~\citep{shen2024tag}               & 41.6 & 64.6 & 89.2 & 86.8 & 42.7 & 66.1 & 89.8 & 87.6 & 43.3 & 67.5 & 90.5 & 88.2 \\
\midrule
\textit{Data Selection} & & & & & & & & & & & & \\
INSTRUCTMINING~\citep{cao2024instruction} & 41.8 & 64.5 & 89.3 & 86.9 & 42.8 & 66.0 & 89.9 & 87.7 & 43.4 & 67.6 & 90.6 & 88.3 \\
S2L~\citep{yang2024smalltolarge}          & 41.9 & 64.7 & 89.4 & 87.0 & 42.9 & 66.2 & 90.0 & 87.8 & 43.5 & 67.8 & 90.7 & 88.4 \\
\midrule
\textbf{PMS-FTP (Ours)}                   & \textbf{42.5} & \textbf{65.5} & \textbf{89.7} & \textbf{87.3} & \textbf{43.4} & \textbf{67.2} & \textbf{90.2} & \textbf{88.0} & \textbf{44.2} & \textbf{69.1} & \textbf{91.1} & \textbf{89.0} \\
\bottomrule
\end{tabular}
}
\end{table*}

\paragraph{Baselines}
We compare PMS-FTP against the following baselines:  
1) \emph{Base Methods}: Full Fine-Tuning (FULL), Fixed Backbone (FIXED);  
2) \emph{Domain Adaptation}: Multi-Domain Adversarial Network (MDAN)~\citep{pei2018multiadversarial}, Moment Matching (M$^3$SDA)~\citep{peng2019moment}, Bayesian Gaussian Mixture (GMDI)~\citep{ling2024bayesian};  
3) \emph{Single-Domain LLM Fine-tuning}: LoRA~\citep{hu2021lora}, Adapter~\citep{houlsby2019parameter}, LLaMA-Adapter~\citep{zhang2024llama}, Q-LoRA~\citep{dettmers2023qlora}, Tag-LLM~\citep{shen2024tag};  
4) \emph{Data Selection}: INSTRUCTMINING (IT)~\citep{cao2024instruction}, S2L~\citep{yang2024smalltolarge}.
Please refer to \ref{app:base} for more details on the baselines. 

\subsection{Experimental Results}
\label{sec:exp-results}
\paragraph{Overall Comparison.}
Table~\ref{tab:results-main} summarizes the test-set performance for each model and method on all four tasks. PMS-FTP consistently surpasses baseline methods across all tasks and model sizes. Compared with strong data-selection (S2L) and single-domain adapter methods (LLaMA-Adapter, Tag-LLM), PMS-FTP achieves improvements by strategically exploiting domain synergies and mitigating negative interference. 
On LLaMA2-13B vs.\ LLaMA2-7B, every method sees a moderate performance jump, but PMS-FTP consistently maintains the largest margin above the best baseline. Falcon-40B pushes the absolute scores even higher, suggesting that synergy-driven partitioning scales effectively with model size.
Table \ref{tab:additional-analysis} in Appendix \ref{app:addi_ex} presents additional experimental results, analyzing why conventional DA baselines lag behind FULL.

\paragraph{Domain-specific performance improvements.}  
We analyze domain synergies and discrepancies (Table~\ref{tab:domain-specific-analysis}). High synergy pairs (\textit{e.g.}, NSum \& Q\&A, Sent \& Q\&A) show substantial gains (+1.8\%, +1.7\%), indicating effective leveraging of complementary domains. Moderate synergy pairs (\textit{e.g.}, Sent \& Topic) also show meaningful improvements (+1.3\%), while even high-discrepancy pairs (\textit{e.g.}, NSum \& Topic) achieve modest gains (+0.9\%). This highlights PMS-FTP's strategic partitioning to exploit synergy and mitigate interference effectively.

\begin{table}[t]
\centering
\caption{Domain-specific performance improvements (LLaMA2-13B backbone).}
\label{tab:domain-specific-analysis}
\resizebox{\linewidth}{!}{
\begin{tabular}{lccccc}
\toprule
\textbf{Domain Grouping} & \textbf{Synergy Score} & \textbf{Discrepancy Score} & \textbf{Avg. Baseline} & \textbf{PMS-FTP} & \textbf{Performance Gain (\%)} \\
\midrule
NSum \& Q\&A                 & \textbf{0.88 (High)}  & 0.12 (Low)    & 64.3 & 66.1 & \textbf{+1.8\%} \\
Sent \& Q\&A                   & 0.85 (High) & 0.15 (Low)    & 89.5 & 91.2 & +1.7\% \\
Q\&A \& Topic                  & 0.80 (High) & 0.20 (Low)    & 76.7 & 78.3 & +1.6\% \\
Sent \& Topic                 & 0.65 (Moderate) & 0.30 (Moderate) & 88.1 & 89.4 & +1.3\% \\
NSum \& Sent                  & 0.60 (Moderate) & 0.40 (Moderate) & 65.2 & 66.4 & +1.2\% \\
Q\&A \& Sent \& Topic      & 0.58 (Moderate) & 0.42 (Moderate) & 77.8 & 79.0 & +1.2\% \\
NSum \& Topic                & 0.40 (Low) & 0.60 (High) & 64.6 & 65.5 & +0.9\% \\
Sent \& NSum \& Topic      & 0.35 (Low) & 0.65 (High) & 80.5 & 81.3 & +0.8\% \\
\bottomrule
\end{tabular}
}
\end{table}
\vspace{-0.1cm}

\begin{wrapfigure}{r}{0.45\linewidth}
    \centering
    \includegraphics[width=0.9\linewidth]{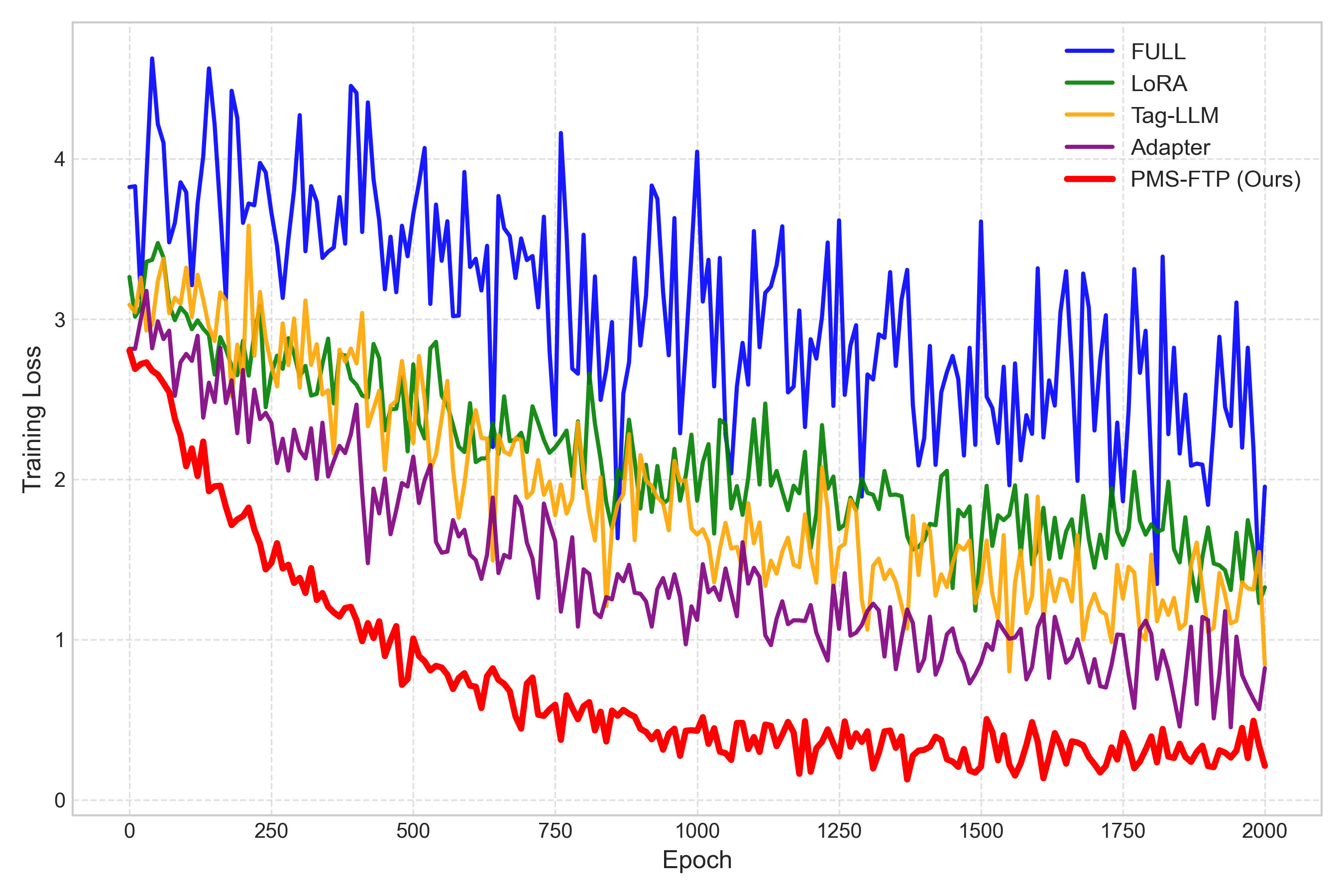}
    \caption{Training loss curves on the Q\&A domain with LLaMA2-13B.}
    \label{fig:loss-convergence}
\end{wrapfigure}
\paragraph{Loss Analysis.}
Figure~\ref{fig:loss-convergence} illustrates the training-loss curves on the Q\&A domain using LLaMA2-13B for several representative methods (FULL, LoRA, LM, Adapter, and our PMS-FTP). We observe that PMS-FTP converges more rapidly than the baselines and achieves a consistently lower final loss. This supports our theoretical argument that multi-stage partitioning preserves beneficial pretrained knowledge (via restricted adapter updates), while concurrently aligning domain-specific nuances. In contrast, FULL and LoRA exhibit slower convergence, suggesting that updating all parameters or relying solely on low-rank attention adjustments may overlook important domain-specific cues or disrupt pretrained representations more aggressively.

\begin{wraptable}{r}{0.6\linewidth}
  \centering
  \caption{Peak GPU memory (in GB) during Q\&A fine-tuning on LLaMA2-7B.}
  \label{tab:peak-memory}
  \resizebox{\linewidth}{!}{ % 宽度自适应，高度按比例缩放
    \footnotesize
    \begin{tabular}{lcc}
      \toprule
      \textbf{Method} & \textbf{Peak GPU (GB)} & \textbf{Relative Reduction} \\
      \midrule
      FULL           & 27.2                 & --                         \\
      LoRA           & 19.6                 & 27.9\%                     \\
      LLaMA-Adapter         & 17.2                 & 37.1\%                     \\
      \textbf{PMS-FTP (Ours)} & 18.4      & 32.4\%                     \\
      \bottomrule
    \end{tabular}
  }
\end{wraptable}

\textbf{Memory Usage.}
Table~\ref{tab:peak-memory} reports the peak allocated memory (in GB) during fine-tuning on the Q\&A task with LLaMA2-7B. We measure memory usage using a single NVIDIA A100 GPU. The \textsc{FULL} method requires the largest memory footprint due to updating all model parameters. \textsc{LoRA} and \textsc{LLaMA-Adapter} both yield substantial savings via sparse or low-rank updates. Our \textsc{PMS-FTP} approach limits the backbone and adapter updates in each stage, keeping overall memory usage about $32\%$ lower than full fine-tuning, albeit slightly higher than LLaMA-Adapter. Nonetheless, the stronger accuracy (see Table~\ref{tab:results-main}) indicates a favorable trade-off between memory efficiency and final performance.

\subsection{Ablation Study}
\label{subsec:ablation}
We further examine \emph{how} each design choice in PMS-FTP impacts final performance. Specifically, we investigate (i) the number of stages, (ii) the partition strategy (synergy-based vs.\ random), (iii) the effect of limiting update norms (i.e., $\|\Delta\theta\|_2 \le \rho_\theta, \|\phi_j\|_2 \le \rho_\phi$), and (iv) synergy metric sensitivity. Experiments in this section use the \emph{LLaMA2-7B} backbone and evaluate on a subset of domains (NSum and Q\&A) for brevity.

\paragraph{Effect of Number of Stages ($M$).}
\begin{wraptable}{r}{0.65\linewidth} % 靠右，宽度调整为65%行宽
  \centering
  \caption{Ablation on the number of stages ($M$) and partition strategies. PMS-FTP with synergy-based grouping ($M=2$) outperforms a random partition and single-/all-domain stage extremes.}
  \label{tab:ablation-stages}
  \resizebox{\linewidth}{!}{ % 自适应缩放
    \small % 可选：进一步缩小字体
    \begin{tabular}{@{}lccc@{}}
      \toprule
      \textbf{Setting} & \textbf{Partition Strategy} & \textbf{NSum (ROUGE-L)} & \textbf{Q\&A (EM)} \\
      \midrule
      $M=1$            & All domains in one stage    & 41.2                   & 63.1               \\
      $M=2$ (Random)   & Random domain grouping     & 41.7                   & 63.9               \\
      $M=2$ (Synergy)  & \textbf{Synergy-driven}    & \textbf{42.2}          & \textbf{64.8}      \\
      $M=4$            & One domain per stage       & 42.0                   & 64.2               \\
      \bottomrule
    \end{tabular}%
  }
\end{wraptable}
In Table~\ref{tab:ablation-stages}, we compare $M=1$ (single-stage updates), $M=2$, and $M=4$ (one stage per domain). We also include a \textit{random} domain grouping for $M=2$ to illustrate the importance of synergy-driven partitioning. Specifically, we measure ROUGE-L on NSum and EM on Q\&A. Single-stage ($M=1$) fine-tuning, akin to multi-task learning without adapter updates, underperforms on both tasks. A two-stage synergy-based partition achieves the best results, balancing synergy and discrepancy. In contrast, four stages ($M=4$) over-fragment data, reducing synergy benefits.

\begin{wraptable}{r}{0.55\linewidth} % 靠右环绕，宽度65%行宽
  \centering
  \footnotesize % 小字体
  \caption{Restricting update norms improves stability and performance. We report average scores (ROUGE-L for NSum, EM for Q\&A).}
  \label{tab:ablation-norms}
  \resizebox{\linewidth}{!}{ % 自适应缩放
    \begin{tabular}{cc|cc}
      \toprule
      $\rho_\theta$ & $\rho_\phi$ & \textbf{NSum (ROUGE-L)} & \textbf{Q\&A (EM)} \\
      \midrule
      0.05 & 0.05 & 41.2 & 63.9 \\
      0.05 & 0.10 & 41.7 & 64.5 \\
      0.05 & 0.20 & 41.5 & 64.1 \\
      \hline
      0.10 & 0.05 & 41.6 & 64.3 \\
      0.10 & 0.10 & \textbf{42.2} & \textbf{64.9} \\
      0.10 & 0.20 & 42.0 & 64.6 \\
      \hline
      0.20 & 0.05 & 41.4 & 64.1 \\
      0.20 & 0.10 & 42.1 & 64.7 \\
      0.20 & 0.20 & 42.1 & 64.8 \\
      \bottomrule
    \end{tabular}%
  }
\end{wraptable}
\paragraph{Effect of Norm Constraints ($\rho_\theta, \rho_\phi$).}
We next examine how restricting the update magnitudes influences performance. By default, we set $\|\theta^t - \theta^{t-1}\|_2 \le \rho_\theta$ and $\|\phi_j^t\|_2 \le \rho_\phi$ to preserve the pretrained backbone's inductive bias . In Table~\ref{tab:ablation-norms}, we vary $\rho_\theta$ and $\rho_\phi$ in $\{0.05, 0.1, 0.2\}$, measuring average performance across NSum/Q\&A.
Too small norms (e.g.\ $\rho_\theta=\rho_\phi=0.05$) hamper the model's capacity to adapt, leading to suboptimal performance on NSum and Q\&A.
Larger norms (0.2) let the model deviate more from $\theta^*$ but risk overfitting. Empirically, $(\rho_\theta,\rho_\phi)=(0.1,0.1)$ or $(0.1,0.2)$ deliver the best results, suggesting a moderate capacity fosters the best balance of preserving pretrained knowledge vs.\ domain-specific adaptation.

\begin{wrapfigure}{r}{0.5\linewidth} % r 表示靠右，0.7\linewidth 控制图片宽度
    \centering
    \vspace{-3mm} % 调整图片垂直位置（可选）
    \includegraphics[width=\linewidth]{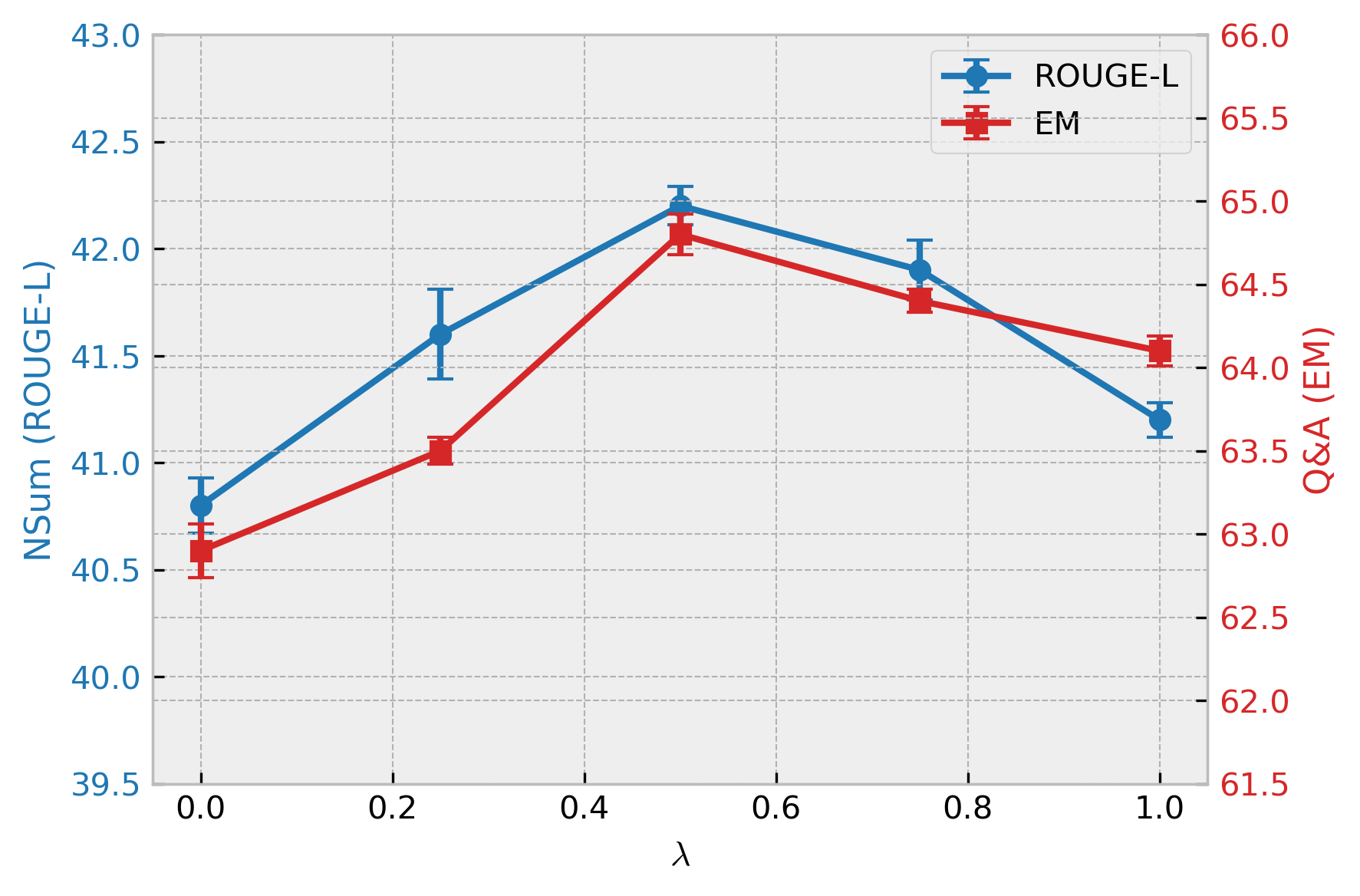} 
    \vspace{-0.5cm}
    \caption{Synergy metric sensitivity.}
    \label{fig:synergy-sensitivity}
\end{wrapfigure}
\paragraph{Synergy Metric Sensitivity.}
Our partition-based approach applies a synergy coefficient $\lambda$ to balance domain synergy against discrepancy (Equation~\eqref{eq:G-def}). We vary $\lambda \in \{0.0,0.25,0.5,0.75,1.0\}$ to observe its impact on partitioning and final performance. Figure~\ref{fig:synergy-sensitivity} (LLaMA2-7B, $M=2$ stages) reports ROUGE-L (NSum) and EM (Q\&A). When $\lambda=0.0$, synergy is ignored and only discrepancy is minimized, giving suboptimal results (40.8 ROUGE-L). Excessively large $\lambda$ (e.g.\ $1.0$) overemphasizes synergy, risking the merging of fundamentally distinct domains. A moderate $\lambda \in [0.25,0.50]$ achieves the best trade-off, aligning with Theorem~\ref{thm:advanced-partition}.

\section{Conclusion}
\label{sec:conclusion}
In this work, we introduced a \textit{partition-based multi-stage fine-tuning} framework to systematically address multi-domain adaptation in large language models. By quantifying each domain's \emph{discrepancy} and \emph{synergy} and jointly optimizing a partition objective, we balance shared feature learning with domain-specific specialization. Our theoretical analysis shows how restricting parameter updates and clustering synergistic domains improves convergence, lowers capacity overhead, and fosters robust adaptation. Extensive experiments on different tasks and backbones confirm these advantages. 

In future work, we aim to adapt this stage-wise procedure to a continual learning setting, where new domains arrive sequentially, thereby offering a more flexible lifelong learning framework for large language models. Furthermore, combining the proposed method with pruning strategies~\citep{lu2024generic,zhou2025dropping,li2025sepprune} represents an interesting direction, especially when dealing with models that have a massive number of parameters.

\section*{Acknowledgements} 
We would like to thank Hunan Airon Technology Co., Ltd. for providing data preprocessing services and computing resources.

%%%%%%

\bibliographystyle{plainnat} 
\bibliography{main}

@phdthesis{blitzer2008domain,
  title={Domain adaptation of natural language processing systems},
  author={Blitzer, John},
  year={2008},
  school={University of Pennsylvania}
}

@article{bartlett2002rademacher,
  title={Rademacher and Gaussian complexities: Risk bounds and structural results},
  author={Bartlett, Peter L and Mendelson, Shahar},
  journal={Journal of Machine Learning Research},
  volume={3},
  number={Nov},
  pages={463--482},
  year={2002}
}

@article{narayan2018don,
  title={Don't give me the details, just the summary! topic-aware convolutional neural networks for extreme summarization},
  author={Narayan, Shashi and Cohen, Shay B and Lapata, Mirella},
  journal={arXiv preprint arXiv:1808.08745},
  year={2018}
}

@article{hermann2015teaching,
  title={Teaching machines to read and comprehend},
  author={Hermann, Karl Moritz and Kocisky, Tomas and Grefenstette, Edward and Espeholt, Lasse and Kay, Will and Suleyman, Mustafa and Blunsom, Phil},
  journal={Advances in neural information processing systems},
  volume={28},
  year={2015}
}

@article{almazrouei2023falcon,
  title={The falcon series of open language models},
  author={Almazrouei, Ebtesam and Alobeidli, Hamza and Alshamsi, Abdulaziz and Cappelli, Alessandro and Cojocaru, Ruxandra and Debbah, M{\'e}rouane and Goffinet, {\'E}tienne and Hesslow, Daniel and Launay, Julien and Malartic, Quentin and others},
  journal={arXiv preprint arXiv:2311.16867},
  year={2023}
}

@inproceedings{socher2013recursive,
  title={Recursive deep models for semantic compositionality over a sentiment treebank},
  author={Socher, Richard and Perelygin, Alex and Wu, Jean and Chuang, Jason and Manning, Christopher D and Ng, Andrew Y and Potts, Christopher},
  booktitle={Proceedings of the 2013 conference on empirical methods in natural language processing},
  pages={1631--1642},
  year={2013}
}

@article{touvron2023llama,
  title={Llama 2: Open foundation and fine-tuned chat models},
  author={Touvron, Hugo and Martin, Louis and Stone, Kevin and Albert, Peter and Almahairi, Amjad and Babaei, Yasmine and Bashlykov, Nikolay and Batra, Soumya and Bhargava, Prajjwal and Bhosale, Shruti and others},
  journal={arXiv preprint arXiv:2307.09288},
  year={2023}
}

@article{neyshabur2017exploring,
  title={Exploring generalization in deep learning},
  author={Neyshabur, Behnam and Bhojanapalli, Srinadh and McAllester, David and Srebro, Nati},
  journal={Advances in neural information processing systems},
  volume={30},
  year={2017}
}

@article{royer2024scalarization,
  title={Scalarization for multi-task and multi-domain learning at scale},
  author={Royer, Amelie and Blankevoort, Tijmen and Ehteshami Bejnordi, Babak},
  journal={Advances in Neural Information Processing Systems},
  volume={36},
  year={2024}
}

@inproceedings{zhang2024m3oe,
  title={M3oe: Multi-domain multi-task mixture-of experts recommendation framework},
  author={Zhang, Zijian and Liu, Shuchang and Yu, Jiaao and Cai, Qingpeng and Zhao, Xiangyu and Zhang, Chunxu and Liu, Ziru and Liu, Qidong and Zhao, Hongwei and Hu, Lantao and others},
  booktitle={Proceedings of the 47th International ACM SIGIR Conference on Research and Development in Information Retrieval},
  pages={893--902},
  year={2024}
}

@article{sener2018multi,
  title={Multi-task learning as multi-objective optimization},
  author={Sener, Ozan and Koltun, Vladlen},
  journal={Advances in neural information processing systems},
  volume={31},
  year={2018}
}

@article{lu2024fine,
  title={Fine-tuning large language models for domain adaptation: Exploration of training strategies, scaling, model merging and synergistic capabilities},
  author={Lu, Wei and Luu, Rachel K and Buehler, Markus J},
  journal={arXiv preprint arXiv:2409.03444},
  year={2024}
}

@article{brown2020language,
  title={Language models are few-shot learners},
  author={Brown, Tom and Mann, Benjamin and Ryder, Nick and Subbiah, Melanie and Kaplan, Jared D and Dhariwal, Prafulla and Neelakantan, Arvind and Shyam, Pranav and Sastry, Girish and Askell, Amanda and others},
  journal={Advances in neural information processing systems},
  volume={33},
  pages={1877--1901},
  year={2020}
}

@article{van2023radadapt,
  title={RadAdapt: Radiology report summarization via lightweight domain adaptation of large language models},
  author={Van Veen, Dave and Van Uden, Cara and Attias, Maayane and Pareek, Anuj and Bluethgen, Christian and Polacin, Malgorzata and Chiu, Wah and Delbrouck, Jean-Benoit and Chaves, Juan Manuel Zambrano and Langlotz, Curtis P and others},
  journal={arXiv preprint arXiv:2305.01146},
  year={2023}
}

@article{zheng2024fine,
  title={Fine-tuning Large Language Models for Domain-specific Machine Translation},
  author={Zheng, Jiawei and Hong, Hanghai and Wang, Xiaoli and Su, Jingsong and Liang, Yonggui and Wu, Shikai},
  journal={arXiv preprint arXiv:2402.15061},
  year={2024}
}

@article{seabra2024contrato360,
  title={Contrato360 2.0: A Document and Database-Driven Question-Answer System using Large Language Models and Agents},
  author={Seabra, Antony and Cavalcante, Claudio and Nepomuceno, Joao and Lago, Lucas and Ruberg, Nicolaas and Lifschitz, Sergio},
  journal={arXiv preprint arXiv:2412.17942},
  year={2024}
}

@inproceedings{peng2019moment,
  title={Moment matching for multi-source domain adaptation},
  author={Peng, Xingchao and Bai, Qinxun and Xia, Xide and Huang, Zijun and Saenko, Kate and Wang, Bo},
  booktitle={Proceedings of the IEEE/CVF international conference on computer vision},
  pages={1406--1415},
  year={2019}
}

@inproceedings{yang2024mentallama,
  title={MentaLLaMA: interpretable mental health analysis on social media with large language models},
  author={Yang, Kailai and Zhang, Tianlin and Kuang, Ziyan and Xie, Qianqian and Huang, Jimin and Ananiadou, Sophia},
  booktitle={Proceedings of the ACM on Web Conference 2024},
  pages={4489--4500},
  year={2024}
}

@article{thirunavukarasu2023large,
  title={Large language models in medicine},
  author={Thirunavukarasu, Arun James and Ting, Darren Shu Jeng and Elangovan, Kabilan and Gutierrez, Laura and Tan, Ting Fang and Ting, Daniel Shu Wei},
  journal={Nature medicine},
  volume={29},
  number={8},
  pages={1930--1940},
  year={2023},
  publisher={Nature Publishing Group US New York}
}

@inproceedings{li2024driving,
  title={Driving everywhere with large language model policy adaptation},
  author={Li, Boyi and Wang, Yue and Mao, Jiageng and Ivanovic, Boris and Veer, Sushant and Leung, Karen and Pavone, Marco},
  booktitle={Proceedings of the IEEE/CVF Conference on Computer Vision and Pattern Recognition},
  pages={14948--14957},
  year={2024}
}

@article{wu2022autoformalization,
  title={Autoformalization with large language models},
  author={Wu, Yuhuai and Jiang, Albert Qiaochu and Li, Wenda and Rabe, Markus and Staats, Charles and Jamnik, Mateja and Szegedy, Christian},
  journal={Advances in Neural Information Processing Systems},
  volume={35},
  pages={32353--32368},
  year={2022}
}

@inproceedings{karanikolas2023large,
  title={Large language models versus natural language understanding and generation},
  author={Karanikolas, Nikitas and Manga, Eirini and Samaridi, Nikoletta and Tousidou, Eleni and Vassilakopoulos, Michael},
  booktitle={Proceedings of the 27th Pan-Hellenic Conference on Progress in Computing and Informatics},
  pages={278--290},
  year={2023}
}

@article{tong2025renaissance,
  title={Renaissance of RNNs in Streaming Clinical Time Series: Compact Recurrence Remains Competitive with Transformers},
  author={Tong, Ran and Liu, Jiaqi and Liu, Su and Hu, Xin and Wang, Lanruo},
  journal={arXiv preprint arXiv:2510.16677},
  year={2025}
}

@article{lin2025cec,
  title={CEC-Zero: Zero-Supervision Character Error Correction with Self-Generated Rewards},
  author={Lin, Zhiming and Zhao, Kai and Zhang, Sophie and Yu, Peilai and Xiao, Canran},
  journal={arXiv preprint arXiv:2512.23971},
  year={2025}
}

@article{jiaqi2025lightweight,
  title={Lightweight baselines for medical abstract classification: DistilBERT with cross-entropy as a strong default},
  author={Liu, Jiaqi and Wang, Tong and Liu, Su and Hu, Xin and Tong, Ran and Wang, Lanruo and Xu, Jiexi},
  journal={arXiv preprint arXiv:2510.10025},
  year={2025}
}

@inproceedings{tan2025profix,
  title     = {ProfiX: Improving Profile-Guided Optimization in Compilers with Graph Neural Networks},
  author    = {Tan, Huiri and Jiang, Juyong and Shen, Jiasi},
  booktitle = {Advances in Neural Information Processing Systems},
  year      = {2025},
  url       = {https://neurips.cc/virtual/2025/poster/119293}
}

@inproceedings{10.1145/3664647.3681115, author = {Jiang, Yiyang and Zhang, Wengyu and Zhang, Xulu and Wei, Xiao-Yong and Chen, Chang Wen and Li, Qing}, title = {Prior Knowledge Integration via LLM Encoding and Pseudo Event Regulation for Video Moment Retrieval}, year = {2024}, isbn = {9798400706868}, publisher = {Association for Computing Machinery}, address = {New York, NY, USA}, url = {https://doi.org/10.1145/3664647.3681115}, doi = {10.1145/3664647.3681115}, booktitle = {Proceedings of the 32nd ACM International Conference on Multimedia}, pages = {7249–7258}, numpages = {10}, location = {Melbourne VIC, Australia}, series = {MM '24} }

@article{zhang2024cf,
  title={CF-DAN: Facial-expression recognition based on cross-fusion dual-attention network},
  author={Zhang, Fan and Chen, Gongguan and Wang, Hua and Zhang, Caiming},
  journal={Computational Visual Media},
  volume={10},
  number={3},
  pages={593--608},
  year={2024},
  publisher={TUP}
}

@article{zhang2023multi,
  title={Multi-scale video super-resolution transformer with polynomial approximation},
  author={Zhang, Fan and Chen, Gongguan and Wang, Hua and Li, Jinjiang and Zhang, Caiming},
  journal={IEEE Transactions on Circuits and Systems for Video Technology},
  volume={33},
  number={9},
  pages={4496--4506},
  year={2023},
  publisher={IEEE}
}

@article{wang2024computing,
  title={Computing nodes for plane data points by constructing cubic polynomial with constraints},
  author={Wang, Hua and Zhang, Fan},
  journal={Computer Aided Geometric Design},
  volume={111},
  pages={102308},
  year={2024},
  publisher={Elsevier}
}

@article{wang2025medical,
  title={A Medical image segmentation model with auto-dynamic convolution and location attention mechanism},
  author={Wang, Yuenan and Wang, Hua and Zhang, Fan},
  journal={Computer Methods and Programs in Biomedicine},
  volume={261},
  pages={108593},
  year={2025},
  publisher={Elsevier}
}

@article{jiang2025transforming,
  title={Transforming time and space: efficient video super-resolution with hybrid attention and deformable transformers},
  author={Jiang, Linling and Wang, Xin and Zhang, Fan and Zhang, Caiming},
  journal={The Visual Computer},
  pages={1--12},
  year={2025},
  publisher={Springer}
}

@article{tao2023dudb,
  title={DUDB: deep unfolding-based dual-branch feature fusion network for pan-sharpening remote sensing images},
  author={Tao, Hailin and Li, Jinjiang and Hua, Zhen and Zhang, Fan},
  journal={IEEE Transactions on Geoscience and Remote Sensing},
  volume={62},
  pages={1--17},
  year={2023},
  publisher={IEEE}
}

@article{xu2025affordance,
  title={Affordance-First Decomposition for Continual Learning in Video-Language Understanding},
  author={Xu, Mengzhu and Liu, Hanzhi and Peng, Ningkang and Chen, Qianyu and Xiao, Canran},
  journal={arXiv preprint arXiv:2512.00694},
  year={2025}
}

@inproceedings{hu-etal-2025-removal,
    title = "Removal of Hallucination on Hallucination: Debate-Augmented {RAG}",
    author = "Hu, Wentao  and
      Zhang, Wengyu  and
      Jiang, Yiyang  and
      Zhang, Chen Jason  and
      Wei, Xiaoyong  and
      Qing, Li",
    editor = "Che, Wanxiang  and
      Nabende, Joyce  and
      Shutova, Ekaterina  and
      Pilehvar, Mohammad Taher",
    booktitle = "Proceedings of the 63rd Annual Meeting of the Association for Computational Linguistics (Volume 1: Long Papers)",
    month = jul,
    year = "2025",
    address = "Vienna, Austria",
    publisher = "Association for Computational Linguistics",
    url = "https://aclanthology.org/2025.acl-long.770/",
    doi = "10.18653/v1/2025.acl-long.770",
    pages = "15839--15853",
    ISBN = "979-8-89176-251-0",
}

@article{yao2024swift,
  title={Swift sampler: Efficient learning of sampler by 10 parameters},
  author={Yao, Jiawei and Li, Chuming and Xiao, Canran},
  journal={Advances in Neural Information Processing Systems},
  volume={37},
  pages={59030--59053},
  year={2024}
}

@inproceedings{yao2023ndc,
  title={Ndc-scene: Boost monocular 3d semantic scene completion in normalized device coordinates space},
  author={Yao, Jiawei and Li, Chuming and Sun, Keqiang and Cai, Yingjie and Li, Hao and Ouyang, Wanli and Li, Hongsheng},
  booktitle={2023 IEEE/CVF International Conference on Computer Vision (ICCV)},
  pages={9421--9431},
  year={2023},
  organization={IEEE Computer Society}
}

@inproceedings{zhang2025enhancing,
  title={Enhancing multimodal large language models complex reason via similarity computation},
  author={Zhang, Xiaofeng and Zeng, Fanshuo and Quan, Yihao and Hui, Zheng and Yao, Jiawei},
  booktitle={Proceedings of the AAAI Conference on Artificial Intelligence},
  volume={39(10)},
  pages={10203--10211},
  year={2025}
}

@article{ke2025early,

  title={Early warning of cryptocurrency reversal risks via multi-source data},

  author={Ke, Zong and Cao, Yuqing and Chen, Zhenrui and Yin, Yuchen and He, Shouchao and Cheng, Yu},

  journal={Finance Research Letters},

  pages={107890},

  year={2025},

  publisher={Elsevier}

}

@inproceedings{xiao2024confusion,
  title={Confusion-resistant federated learning via diffusion-based data harmonization on non-IID data},
  author={Chen, Xiaohong and Xiao, Canran and Liu, Yongmei},
  booktitle={Proceedings of the 38th International Conference on Neural Information Processing Systems},
  pages={137495--137520},
  year={2024}
}

@article{xiao2025multifrequency,
  title   = {A multifrequency data fusion deep learning model for carbon price prediction},
  author  = {Xiao, Canran and Liu, Yongmei},
  journal = {Journal of Forecasting},
  year    = {2025},
  volume  = {44},
  number  = {2},
  pages   = {436--458},
}

@article{lu20254d,
  title={4d driving scene generation with stereo forcing},
  author={Lu, Hao and Ma, Zhuang and Jiang, Guangfeng and Ge, Wenhang and Li, Bohan and Cai, Yuzhan and Zheng, Wenzhao and Zhang, Yunpeng and Chen, Yingcong},
  journal={arXiv preprint arXiv:2509.20251},
  year={2025}
}

@article{zeng2025janusvln,
            title={JanusVLN: Decoupling Semantics and Spatiality with Dual Implicit Memory for Vision-Language Navigation},
            author={Zeng, Shuang and Qi, Dekang and Chang, Xinyuan and Xiong, Feng and Xie, Shichao and Wu, Xiaolong and Liang, Shiyi and Xu, Mu and Wei, Xing},
            journal={arXiv preprint arXiv:2509.22548},
            year={2025}
            }

@inproceedings{zhang2025omniguard,
  title={Omniguard: Hybrid manipulation localization via augmented versatile deep image watermarking},
  author={Zhang, Xuanyu and Tang, Zecheng and Xu, Zhipei and Li, Runyi and Xu, Youmin and Chen, Bin and Gao, Feng and Zhang, Jian},
  booktitle={Proceedings of the Computer Vision and Pattern Recognition Conference},
  pages={3008--3018},
  year={2025}
}

@inproceedings{DePro,
author = {Chen, Kaixiang and Fang, Pengfei and Xue, Hui},
title = {DePro: Domain Ensemble using Decoupled Prompts for Universal Cross-Domain Retrieval},
year = {2025},
booktitle = {Proceedings of the 48th International ACM SIGIR Conference on Research and Development in Information Retrieval},
pages = {958–967},
numpages = {10},
location = {Padua, Italy},
series = {SIGIR '25}
}

@article{zeng2025FSDrive,
      title={FutureSightDrive: Thinking Visually with Spatio-Temporal CoT for Autonomous Driving},
      author={Shuang Zeng and Xinyuan Chang and Mengwei Xie and Xinran Liu and Yifan Bai and Zheng Pan and Mu Xu and Xing Wei},
      journal={arXiv preprint arXiv:2505.17685},
      year={2025}
      }

@inproceedings{zhang2025yoloppa,
  title     = {YOLO-PPA based Efficient Traffic Sign Detection for Cruise Control in Autonomous Driving},
  author    = {Zhang, Jingyu and Zhang, Wenqing and Tan, Chaoyi and Li, Xiangtian and Sun, Qianyi},
  booktitle = {Proceedings of the 2024 International Conference on Industrial Automation and Robotics (IAR)},
  year      = {2025},
  pages     = {8--13},
  publisher = {ACM},
  doi       = {10.1145/3707402.3707404},
  url       = {https://doi.org/10.1145/3707402.3707404}
}

@inproceedings{10.1145/3711896.3737195,
author = {Sun, Yu and Li, Yin and Sun, Ruixiao and Liu, Chunhui and Zhou, Fangming and Jin, Ze and Wang, Linjie and Shen, Xiang and Hao, Zhuolin and Xiong, Hongyu},
title = {Audio-Enhanced Vision-Language Modeling with Latent Space Broadening for High Quality Data Expansion},
year = {2025},
isbn = {9798400714542},
publisher = {Association for Computing Machinery},
address = {New York, NY, USA},
url = {https://doi.org/10.1145/3711896.3737195},
doi = {10.1145/3711896.3737195},
booktitle = {Proceedings of the 31st ACM SIGKDD Conference on Knowledge Discovery and Data Mining V.2},
pages = {4872–4881},
numpages = {10},
location = {Toronto ON, Canada},
series = {KDD '25}
}

@article{He2025UnifiedMetric,
  title={A Unified Metric Architecture for AI Infrastructure: A Cross-Layer Taxonomy Integrating Performance, Efficiency, and Cost},
  author={He, Qi},
  journal={arXiv preprint arXiv:2511.21772},
  year={2025}
}

@article{lu2024generic,
  title={A generic layer pruning method for signal modulation recognition deep learning models},
  author={Lu, Yao and Zhu, Yutao and Li, Yuqi and Xu, Dongwei and Lin, Yun and Xuan, Qi and Yang, Xiaoniu},
  journal={IEEE Transactions on Cognitive Communications and Networking},
  year={2024},
  publisher={IEEE}
}

@inproceedings{lu2022understanding,
  title={Understanding the dynamics of dnns using graph modularity},
  author={Lu, Yao and Yang, Wen and Zhang, Yunzhe and Chen, Zuohui and Chen, Jinyin and Xuan, Qi and Wang, Zhen and Yang, Xiaoniu},
  booktitle={European Conference on Computer Vision},
  pages={225--242},
  year={2022},
  organization={Springer}
}

@article{li2025ammkd,
  title={AMMKD: Adaptive Multimodal Multi-teacher Distillation for Lightweight Vision-Language Models},
  author={Li, Yuqi and Yang, Chuanguang and Dong, Junhao and Yao, Zhengtao and Xu, Haoyan and Dong, Zeyu and Zeng, Hansheng and An, Zhulin and Tian, Yingli},
  journal={arXiv preprint arXiv:2509.00039},
  year={2025}
}

@article{li2025sepprune,
  title={Sepprune: Structured pruning for efficient deep speech separation},
  author={Li, Yuqi and Li, Kai and Yin, Xin and Yang, Zhifei and Dong, Junhao and Dong, Zeyu and Yang, Chuanguang and Tian, Yingli and Lu, Yao},
  journal={arXiv preprint arXiv:2505.12079},
  year={2025}
}

@inproceedings{yan2025hemora,
  title={HeMoRa: Unsupervised Heuristic Consensus Sampling for Robust Point Cloud Registration},
  author={Yan, Shaocheng and Wang, Yiming and Zhao, Kaiyan and Shi, Pengcheng and Zhao, Zhenjun and Zhang, Yongjun and Li, Jiayuan},
  booktitle={Proceedings of the Computer Vision and Pattern Recognition Conference},
  pages={1363--1373},
  year={2025}
}

@article{li2023hong,
  title={Hong kong world: Leveraging structural regularity for line-based slam},
  author={Li, Haoang and Zhao, Ji and Bazin, Jean-Charles and Kim, Pyojin and Joo, Kyungdon and Zhao, Zhenjun and Liu, Yun-Hui},
  journal={IEEE Transactions on Pattern Analysis and Machine Intelligence},
  volume={45},
  number={11},
  pages={13035--13053},
  year={2023},
  publisher={IEEE}
}

@inproceedings{wen2023syreanet,
  author={Wen, Junjie and Cui, Jinqiang and Zhao, Zhenjun and Yan, Ruixin and Gao, Zhi and Dou, Lihua and Chen, Ben M.},
  booktitle={IEEE International Conference on Robotics and Automation (ICRA)}, 
  title={SyreaNet: A Physically Guided Underwater Image Enhancement Framework Integrating Synthetic and Real Images}, 
  year={2023},
  pages={5177-5183}
}

@inproceedings{zhou2025dropping,
  title={Dropping Experts, Recombining Neurons: Retraining-Free Pruning for Sparse Mixture-of-Experts LLMs},
  author={Zhou, Yixiao and Zhao, Ziyu and Cheng, Dongzhou and Wu, Zhiliang and Gui, Jie and Yang, Yi and Wu, Fei and Cheng, Yu and Fan, Hehe},
  booktitle={Findings of the Association for Computational Linguistics: EMNLP 2025},
  pages={15169--15186},
  year={2025}
}

@article{lu2025uniugp,
  title={UniUGP: Unifying Understanding, Generation, and Planing For End-to-end Autonomous Driving},
  author={Lu, Hao and Liu, Ziyang and Jiang, Guangfeng and Luo, Yuanfei and Chen, Sheng and Zhang, Yangang and Chen, Ying-Cong},
  journal={ByteDance Seed Tech Report},
  year={2025}
}

@article{sun2025objective,
  author       = {Sun, Wenxi and Shen, Qiannan and Gao, Yijun and Mao, Qinkai and Qi, Tongsong and Xu, Shuo},
  title        = {Objective over Architecture: Fraud Detection Under Extreme Imbalance in Bank Account Opening},
  journal      = {Computation},
  year         = {2025},
  volume       = {13},
  number       = {12},
  pages        = {290},
  doi          = {10.3390/computation13120290},
  url          = {https://www.mdpi.com/2079-3197/13/12/290}
}

@misc{shen2025aienhanced,
  author = {Shen, Qiannan and Zhang, Jing},
  title = {AI-Enhanced Disaster Risk Prediction with Explainable SHAP Analysis: A Multi-Class Classification Approach Using XGBoost},
  year = {2025},
  month = {December},
  publisher = {Research Square},
  doi = {10.21203/rs.3.rs-8437180/v1},
  url = {https://www.researchsquare.com/article/rs-8437180/v1},
  note = {Preprint, Version 1, posted December 31, 2025}
}

@article{chen2025framework,
  title   = {Framework and Pathway for the Construction of a Unified Data-Element Market in China},
  author  = {Chen, Xiaohong and Xiao, Canran and Cao, Wenzhi and Zhang, Weiwei and Liu, Yongmei},
  journal = {Strategic Study of Chinese Academy of Engineering},
  year    = {2025},
  volume  = {27},
  number  = {1},
  pages   = {40--50}
}

@inproceedings{xiao2025diffusion,
  title     = {Diffusion-Based Self-Supervised Imitation Learning from Imperfect Visual Servoing Demonstrations for Robotic Glass Installation},
  author    = {Xiao, Canran and Hou, Liwei and Fu, Ling and Chen, Wenrui},
  booktitle = {2025 IEEE International Conference on Robotics and Automation (ICRA)},
  year      = {2025},
  pages     = {10401--10407},
  publisher = {IEEE},
}

@article{xiao2025curiosity,
  title   = {Curiosity Meets Cooperation: A Game-Theoretic Approach to Long-Tail Multi-Label Learning},
  author  = {Xiao, Canran and Zhao, Chuangxin and Ke, Zong and Shen, Fei},
  journal = {arXiv preprint arXiv:2510.17520},
  year    = {2025},
}

@article{kumar2024large,
  title={Large language models (LLMs): survey, technical frameworks, and future challenges},
  author={Kumar, Pranjal},
  journal={Artificial Intelligence Review},
  volume={57},
  number={10},
  pages={260},
  year={2024},
  publisher={Springer}
}

@inproceedings{zhang2024llama,
  title={LLaMA-adapter: Efficient fine-tuning of large language models with zero-initialized attention},
  author={Zhang, Renrui and Han, Jiaming and Liu, Chris and Zhou, Aojun and Lu, Pan and Qiao, Yu and Li, Hongsheng and Gao, Peng},
  booktitle={The Twelfth International Conference on Learning Representations},
  year={2024}
}

@inproceedings{pei2018multiadversarial,
  title={Multi-Adversarial Domain Adaptation},
  author={Pei, Zhi and Cao, Zhen and Long, Mingsheng and Wang, Jianmin},
  booktitle={Proceedings of the AAAI Conference on Artificial Intelligence},
  pages={3934--3941},
  year={2018}
}

@inproceedings{ganin2015unsupervised,
  title={Unsupervised Domain Adaptation by Backpropagation},
  author={Ganin, Yaroslav and Lempitsky, Victor},
  booktitle={Proceedings of the 32nd International Conference on Machine Learning (ICML)},
  pages={1180--1189},
  year={2015}
}

@inproceedings{houlsby2019parameter,
  title={Parameter-efficient transfer learning for NLP},
  author={Houlsby, Neil and Giurgiu, Andrei and Jastrzebski, Stanislaw and Morrone, Bruna and De Laroussilhe, Quentin and Gesmundo, Andrea and Attariyan, Mona and Gelly, Sylvain},
  booktitle={International conference on machine learning},
  pages={2790--2799},
  year={2019},
  organization={PMLR}
}

@article{hu2021lora,
  title={Lora: Low-rank adaptation of large language models},
  author={Hu, Edward J and Shen, Yelong and Wallis, Phillip and Allen-Zhu, Zeyuan and Li, Yuanzhi and Wang, Shean and Wang, Lu and Chen, Weizhu},
  journal={arXiv preprint arXiv:2106.09685},
  year={2021}
}

@article{zhang2015character,
  title={Character-level convolutional networks for text classification},
  author={Zhang, Xiang and Zhao, Junbo and LeCun, Yann},
  journal={Advances in neural information processing systems},
  volume={28},
  year={2015}
}

@article{rajpurkar2016squad,
  title={Squad: 100,000+ questions for machine comprehension of text},
  author={Rajpurkar, P},
  journal={arXiv preprint arXiv:1606.05250},
  year={2016}
}

@inproceedings{li2024quantity,
  title={From Quantity to Quality: Boosting LLM Performance with Self-Guided Data Selection for Instruction Tuning},
  author={Li, Ming and Zhang, Yong and Li, Zhitao and Chen, Jiuhai and Chen, Lichang and Cheng, Ning and Wang, Jianzong and Zhou, Tianyi and Xiao, Jing},
  booktitle={Proceedings of the 2024 Conference of the North American Chapter of the Association for Computational Linguistics: Human Language Technologies (Volume 1: Long Papers)},
  pages={7595--7628},
  year={2024}
}

@article{yang2018hotpotqa,
  title={HotpotQA: A dataset for diverse, explainable multi-hop question answering},
  author={Yang, Zhilin and Qi, Peng and Zhang, Saizheng and Bengio, Yoshua and Cohen, William W and Salakhutdinov, Ruslan and Manning, Christopher D},
  journal={arXiv preprint arXiv:1809.09600},
  year={2018}
}

@misc{mohri2018foundations,
  title={Foundations of machine learning. Adaptive computation and machine learning},
  author={Mohri, Mehryar and Rostamizadeh, Afshin and Talwalkar, Ameet},
  year={2018},
  publisher={MIT Press, Cambridge, MA}
}

@article{ling2024bayesian,
  title={Bayesian domain adaptation with gaussian mixture domain-indexing},
  author={Ling, Yanfang and Li, Jiyong and Li, Lingbo and Liang, Shangsong},
  journal={Advances in Neural Information Processing Systems},
  volume={37},
  pages={87226--87254},
  year={2024}
}

@article{yang2024smalltolarge,
  title={Smalltolarge (s2l): Scalable data selection for fine-tuning large language models by summarizing training trajectories of small models},
  author={Yang, Yu and Mishra, Siddhartha and Chiang, Jeffrey and Mirzasoleiman, Baharan},
  journal={Advances in Neural Information Processing Systems},
  volume={37},
  pages={83465--83496},
  year={2024}
}

@article{pang2024improving,
  title={Improving data efficiency via curating llm-driven rating systems},
  author={Pang, Jinlong and Wei, Jiaheng and Shah, Ankit Parag and Zhu, Zhaowei and Wang, Yaxuan and Qian, Chen and Liu, Yang and Bao, Yujia and Wei, Wei},
  journal={arXiv preprint arXiv:2410.10877},
  year={2024}
}

@article{lin2024not,
  title={Not all tokens are what you need for pretraining},
  author={Lin, Zhenghao and Gou, Zhibin and Gong, Yeyun and Liu, Xiao and Xu, Ruochen and Lin, Chen and Yang, Yujiu and Jiao, Jian and Duan, Nan and Chen, Weizhu and others},
  journal={Advances in Neural Information Processing Systems},
  volume={37},
  pages={29029--29063},
  year={2024}
}

@article{gou2024mixed,
  title={Mixed preference optimization: Reinforcement learning with data selection and better reference model},
  author={Gou, Qi and Nguyen, Cam-Tu},
  journal={arXiv preprint arXiv:2403.19443},
  year={2024}
}

@inproceedings{cao2024instruction,
  title     = {Instruction mining: Instruction data selection for tuning large language models},
  author    = {Cao, Yihan and Kang, Yanbin and Wang, Chi and Sun, Lichao},
  booktitle = {First Conference on Language Modeling},
  year      = {2024}
}

@article{dettmers2023qlora,
  title={Qlora: Efficient finetuning of quantized llms},
  author={Dettmers, Tim and Pagnoni, Artidoro and Holtzman, Ari and Zettlemoyer, Luke},
  journal={Advances in neural information processing systems},
  volume={36},
  pages={10088--10115},
  year={2023}
}

@inproceedings{shen2024tag,
  title={Tag-LLM: repurposing general-purpose LLMs for specialized domains},
  author={Shen, Junhong and Tenenholtz, Neil and Hall, James Brian and Alvarez-Melis, David and Fusi, Nicol{\`o}},
  booktitle={Proceedings of the 41st International Conference on Machine Learning},
  pages={44759--44773},
  year={2024}
}

%%%%%%%%%%%%%%%%%%%%%%%%%%%%%%%%%%%%%%%%%%%%%%%%%%%%%%%%%%%%

\newpage
\section*{NeurIPS Paper Checklist}

\begin{enumerate}

\item {\bf Claims}
    \item[] Question: Do the main claims made in the abstract and introduction accurately reflect the paper's contributions and scope?
    \item[] Answer: \answerYes{}
    \item[] Justification: The abstract and introduction clearly and accurately present our contributions, which include a novel partition-based fine-tuning framework, theoretical analysis with generalization bounds, and extensive experimental validation. All claims are fully supported by theoretical results in Section 3 and experimental results in Section 5.

\item {\bf Limitations}
    \item[] Question: Does the paper discuss the limitations of the work performed by the authors?
    \item[] Answer:  \answerYes{}
    \item[] Justification: We discuss the computational overhead associated with domain partitioning and provide empirical insights into scalability.
    \item[] Guidelines:
    \begin{itemize}
        \item The answer NA means that the paper has no limitation while the answer No means that the paper has limitations, but those are not discussed in the paper. 
        \item The authors are encouraged to create a separate ``Limitations'' section in their paper.
        \item The paper should point out any strong assumptions and how robust the results are to violations of these assumptions (e.g., independence assumptions, noiseless settings, model well-specification, asymptotic approximations only holding locally). The authors should reflect on how these assumptions might be violated in practice and what the implications would be.
        \item The authors should reflect on the scope of the claims made, e.g., if the approach was only tested on a few datasets or with a few runs. In general, empirical results often depend on implicit assumptions, which should be articulated.
        \item The authors should reflect on the factors that influence the performance of the approach. For example, a facial recognition algorithm may perform poorly when image resolution is low or images are taken in low lighting. Or a speech-to-text system might not be used reliably to provide closed captions for online lectures because it fails to handle technical jargon.
        \item The authors should discuss the computational efficiency of the proposed algorithms and how they scale with dataset size.
        \item If applicable, the authors should discuss possible limitations of their approach to address problems of privacy and fairness.
        \item While the authors might fear that complete honesty about limitations might be used by reviewers as grounds for rejection, a worse outcome might be that reviewers discover limitations that aren't acknowledged in the paper. The authors should use their best judgment and recognize that individual actions in favor of transparency play an important role in developing norms that preserve the integrity of the community. Reviewers will be specifically instructed to not penalize honesty concerning limitations.
    \end{itemize}

\item {\bf Theory assumptions and proofs}
    \item[] Question: For each theoretical result, does the paper provide the full set of assumptions and a complete (and correct) proof?
    \item[] Answer: \answerYes{}% Replace by \answerYes{}, \answerNo{}, or \answerNA{}.
    \item[] Justification:  All assumptions required for our theoretical results are explicitly stated. Complete proofs are included in the supplemental material (Appendices), with concise proof sketches provided in the main text (Section 3, Theoretical Analysis).

    \item {\bf Experimental result reproducibility}
    \item[] Question: Does the paper fully disclose all the information needed to reproduce the main experimental results of the paper to the extent that it affects the main claims and/or conclusions of the paper (regardless of whether the code and data are provided or not)?
    \item[] Answer: \answerYes{}
    \item[] Justification: The paper provides detailed experimental settings, including dataset descriptions, data splits, hyperparameters, model configurations, and optimization procedures, clearly presented in the main text and appendices.

\item {\bf Open access to data and code}
    \item[] Question: Does the paper provide open access to the data and code, with sufficient instructions to faithfully reproduce the main experimental results, as described in supplemental material?
    \item[] Answer: \answerNo{} % Replace by \answerYes{}, \answerNo{}, or \answerNA{}.
    \item[] Justification: Due to institutional restrictions and proprietary considerations, the data and code used in this study are not publicly available at this time. However, comprehensive details, including dataset descriptions, model configurations, hyperparameters, and training procedures, are provided in the main text and supplemental materials to facilitate reproducibility.
    \item[] Guidelines:
    \begin{itemize}
        \item The answer NA means that paper does not include experiments requiring code.
        \item Please see the NeurIPS code and data submission guidelines (\url{https://nips.cc/public/guides/CodeSubmissionPolicy}) for more details.
        \item While we encourage the release of code and data, we understand that this might not be possible, so ``No'' is an acceptable answer. Papers cannot be rejected simply for not including code, unless this is central to the contribution (e.g., for a new open-source benchmark).
        \item The instructions should contain the exact command and environment needed to run to reproduce the results. See the NeurIPS code and data submission guidelines (\url{https://nips.cc/public/guides/CodeSubmissionPolicy}) for more details.
        \item The authors should provide instructions on data access and preparation, including how to access the raw data, preprocessed data, intermediate data, and generated data, etc.
        \item The authors should provide scripts to reproduce all experimental results for the new proposed method and baselines. If only a subset of experiments are reproducible, they should state which ones are omitted from the script and why.
        \item At submission time, to preserve anonymity, the authors should release anonymized versions (if applicable).
        \item Providing as much information as possible in supplemental material (appended to the paper) is recommended, but including URLs to data and code is permitted.
    \end{itemize}

\item {\bf Experimental setting/details}
    \item[] Question: Does the paper specify all the training and test details (e.g., data splits, hyperparameters, how they were chosen, type of optimizer, etc.) necessary to understand the results?
    \item[] Answer: \answerYes{}
    \item[] Justification: The paper explicitly describes training and test splits, model architectures, choice of hyperparameters, optimization methods, and computational settings in the experimental sections and appendixs.

\item {\bf Experiment statistical significance}
    \item[] Question: Does the paper report error bars suitably and correctly defined or other appropriate information about the statistical significance of the experiments?
    \item[] Answer: \answerYes{}
    \item[] Justification: The paper reports experimental results with clearly defined error bars, calculated as the standard deviation across multiple independent runs.
    \item[] Guidelines:
    \begin{itemize}
        \item The answer NA means that the paper does not include experiments.
        \item The authors should answer ``Yes'' if the results are accompanied by error bars, confidence intervals, or statistical significance tests, at least for the experiments that support the main claims of the paper.
        \item The factors of variability that the error bars are capturing should be clearly stated (for example, train/test split, initialization, random drawing of some parameter, or overall run with given experimental conditions).
        \item The method for calculating the error bars should be explained (closed form formula, call to a library function, bootstrap, etc.)
        \item The assumptions made should be given (e.g., Normally distributed errors).
        \item It should be clear whether the error bar is the standard deviation or the standard error of the mean.
        \item It is OK to report 1-sigma error bars, but one should state it. The authors should preferably report a 2-sigma error bar than state that they have a 96\% CI, if the hypothesis of Normality of errors is not verified.
        \item For asymmetric distributions, the authors should be careful not to show in tables or figures symmetric error bars that would yield results that are out of range (e.g. negative error rates).
        \item If error bars are reported in tables or plots, The authors should explain in the text how they were calculated and reference the corresponding figures or tables in the text.
    \end{itemize}

\item {\bf Experiments compute resources}
    \item[] Question: For each experiment, does the paper provide sufficient information on the computer resources (type of compute workers, memory, time of execution) needed to reproduce the experiments?
    \item[] Answer: \answerYes{}
    \item[] Justification: The paper clearly specifies the computational resources utilized, including GPU type, memory requirements, execution time per run, and overall compute needed for each experimental setting.
    
\item {\bf Code of ethics}
    \item[] Question: Does the research conducted in the paper conform, in every respect, with the NeurIPS Code of Ethics \url{https://neurips.cc/public/EthicsGuidelines}?
    \item[] Answer:  \answerYes{}
    \item[] Justification: We have thoroughly reviewed the NeurIPS Code of Ethics and confirm that our research fully complies with the guidelines.

\item {\bf Broader impacts}
    \item[] Question: Does the paper discuss both potential positive societal impacts and negative societal impacts of the work performed?
    \item[] Answer: \answerNA{}.
    \item[] Justification: There is no societal impact of the work performed.
    
\item {\bf Safeguards}
    \item[] Question: Does the paper describe safeguards that have been put in place for responsible release of data or models that have a high risk for misuse (e.g., pretrained language models, image generators, or scraped datasets)?
    \item[] Answer: \answerNA{}.
    \item[] Justification: The paper poses no such risks.

\item {\bf Licenses for existing assets}
    \item[] Question: Are the creators or original owners of assets (e.g., code, data, models), used in the paper, properly credited and are the license and terms of use explicitly mentioned and properly respected?
    \item[] Answer: \answerYes{}
    \item[] Justification: All datasets and models used in our experiments are properly credited with citations to their original sources.

\item {\bf New assets}
    \item[] Question: Are new assets introduced in the paper well documented and is the documentation provided alongside the assets?
    \item[] Answer: \answerNA{}.
    \item[] Justification: The paper does not release new assets.

\item {\bf Crowdsourcing and research with human subjects}
    \item[] Question: For crowdsourcing experiments and research with human subjects, does the paper include the full text of instructions given to participants and screenshots, if applicable, as well as details about compensation (if any)? 
    \item[] Answer: \answerNA{}
    \item[] Justification: The paper does not involve crowdsourcing nor research with human subjects.

\item {\bf Institutional review board (IRB) approvals or equivalent for research with human subjects}
    \item[] Question: Does the paper describe potential risks incurred by study participants, whether such risks were disclosed to the subjects, and whether Institutional Review Board (IRB) approvals (or an equivalent approval/review based on the requirements of your country or institution) were obtained?
    \item[] Answer: \answerNA{}
    \item[] Justification: The paper does not involve crowdsourcing nor research with human subjects.

\item {\bf Declaration of LLM usage}
    \item[] Question: Does the paper describe the usage of LLMs if it is an important, original, or non-standard component of the core methods in this research? Note that if the LLM is used only for writing, editing, or formatting purposes and does not impact the core methodology, scientific rigorousness, or originality of the research, declaration is not required.
    %this research? 
    \item[] Answer:  \answerNA{}
    \item[] Justification:  In this work, LLMs were employed solely for improving language clarity.

\end{enumerate}

\newpage

\appendix

\section{Supplementary Description of Experimental Setup}
\subsection{Datasets}
\label{app:data}
Table~\ref{tab:dataset_summary} presents the detailed statistics of the datasets used in this work.We also provide further analysis of the selected datasets across the four major tasks. 

\begin{table}[ht]
\centering
\caption{Summary of the multi-domain datasets used in our experiments.}
\label{tab:dataset_summary}
\begin{tabular}{lcccc}
\toprule
\textbf{Dataset} & \textbf{\#Train} & \textbf{\#Val} & \textbf{\#Test} & \textbf{Metric}\\
\midrule
NSum (News Summ.)     & 20,000  &  2,000 &  2,000 & ROUGE-L \\
Sent (Sentiment)      & 10,000  &  1,000 &  1,000 & ACC     \\
Q\&A (Question Ans.)  & 15,000  &  1,500 &  1,500 & EM / F1 \\
Topic (Classification)& 12,000  &  1,200 &  1,200 & ACC     \\
\bottomrule
\end{tabular}
\end{table}

\begin{figure}[htbp]
    \centering
    \includegraphics[width=\textwidth]{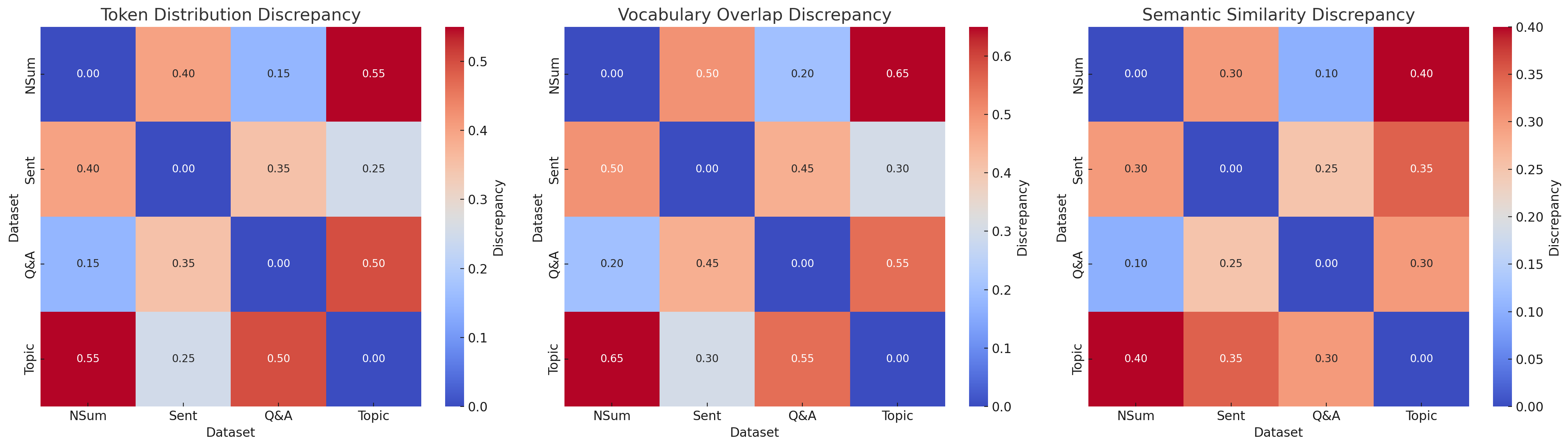}
    \caption{Domain discrepancy heatmaps across three dimensions: Token Distribution, Vocabulary Overlap, and Semantic Similarity.}
    \label{fig:domain-discrepancy}
\end{figure}

As shown in Figure~\ref{fig:domain-discrepancy}, we illustrate the domain discrepancies across three distinct dimensions. In the \textit{Token Distribution} dimension, NSum and Q\&A exhibit a relatively low discrepancy (0.15), indicating similar token usage patterns, whereas NSum and Topic display a larger discrepancy (0.55). In the \textit{Vocabulary Overlap} dimension, NSum and Q\&A share more vocabulary (discrepancy of 0.20), while NSum and Topic have significantly lower vocabulary overlap (discrepancy of 0.65). Regarding the \textit{Semantic Similarity} dimension, NSum and Q\&A show the highest semantic closeness (discrepancy of 0.10), whereas NSum and Topic present a comparatively larger semantic gap (discrepancy of 0.40).

\subsection{Baselines}
\label{app:base}
\paragraph{Baselines.}
To give a fair and transparent point of comparison, we implement or re-run every baseline under \emph{exactly one of two backbone-update protocols}:

1) \textbf{Full FT} - all LLM parameters are updated (\(\sim\!\!100\%\) trainable).  This is the strongest-but most memory-hungry-setting.

2) \textbf{PEFT} - the LLM backbone is \emph{frozen}; only light-weight adapters  
          (LoRA, Houlsby adapters, etc.) or newly added heads are trained  
          (\(<\!1\%\) of parameters).\footnote{Frozen-backbone PEFT has become standard practice in recent  
          parameter-efficient fine-tuning work and, in preliminary tests, performs on par  
          with-or better than-full fine-tuning when data are limited.}

Our PMS-FTP always uses the \underline{PEFT} protocol: the backbone drift per stage is bounded by \(\rho_\theta\), and only domain adapters \(\phi_j\) are newly trained.

The competing methods are grouped as follows:

\begin{itemize}
  \item \textbf{Base Methods}
        \begin{itemize}
            \item \textbf{FULL} (\textit{Full FT}): classical end-to-end fine-tuning.
            \item \textbf{FIXED} (\textit{PEFT}): backbone frozen, task head only.  
        \end{itemize}
        
  \item \textbf{Domain-Adaptation (Full FT)}  
        \begin{itemize}
            \item \textbf{MDAN} \citep{pei2018multiadversarial}: multi-adversarial domain classifiers.  
            \item \textbf{M$^{3}$SDA} \citep{peng2019moment}: moment matching across domains.  
            \item \textbf{GMDI} \citep{ling2024bayesian}: Bayesian Gaussian-mixture domain indexing.  
        \end{itemize}
        \emph{Implementation note:} all DA methods update the entire LLM just as FULL, and their
        extra losses are added on top of the cross-entropy objective.

  \item \textbf{Single-Domain PEFT}
        \begin{itemize}
            \item \textbf{LoRA} \citep{hu2021lora}: low-rank adapters in attention projections.  
            \item \textbf{Adapter} \citep{houlsby2019parameter}: Houlsby bottleneck adapters.  
            \item \textbf{LLaMA-Adapter} \citep{zhang2024llama}: zero-init residual adapters.  
            \item \textbf{Q-LoRA} \citep{dettmers2023qlora}: 4-bit quantised LoRA layers.  
            \item \textbf{Tag-LLM} \citep{shen2024tag}: task-aware gating with soft prompts.  
        \end{itemize}
        These methods \underline{freeze} the backbone; only adapter / prompt parameters are updated.

  \item \textbf{Data-Selection PEFT}
        \begin{itemize}
            \item \textbf{INSTRUCTMINING (IT)} \citep{cao2024instruction}: filters high-quality instructions before LoRA fine-tuning.  
            \item \textbf{S2L} \citep{yang2024smalltolarge}: curriculum ordering via proxy-model clustering; uses LoRA layers.  
        \end{itemize}
\end{itemize}

\subsection{Implementation Details}
\label{sec:implementation}

\paragraph{Hardware and Software.}
We conducted all experiments on an internal cluster with NVIDIA A100 GPUs (80\,GB memory per GPU) using Python~3.9, PyTorch~2.0.0, and HuggingFace Transformers~4.30.2. Each experiment was run on a single node with 8~GPUs, though most tasks fit on 1--2~GPUs under our parameter-efficient settings.

\paragraph{Data Splits and Preprocessing.}
Each dataset is partitioned into train/validation/test splits, as noted in Table~\ref{tab:dataset_summary}. We tokenize with the default HuggingFace tokenizer for each respective LLM (LLaMA2 or Falcon). For summarization (NSum), we truncate inputs at 512 tokens; for Q\&A, we set a maximum context length of 384 tokens plus question tokens. Other tasks are capped at 256 tokens per sample. All special tokens remain as defined in each LLM's tokenizer.

\paragraph{Training Configuration.}
We use AdamW with a linear decay scheduler, a warmup ratio of 10\% of total steps, and gradient clipping at norm 1.0. Table~\ref{tab:hp} gives key hyperparameters. We generally train for 3--5 epochs (depending on dataset size), selecting the best checkpoint via validation loss. Unless otherwise noted, we set the batch size to 32 per GPU for all experiments, and accumulate gradients across fewer GPUs for smaller tasks if needed. We adopt the default mixed-precision (fp16) training in PyTorch.  

\begin{table}[h]
\centering
\caption{Default hyperparameter values.}
\label{tab:hp}
\small
\begin{tabular}{lc}
\toprule
\textbf{Hyperparameter} & \textbf{Value} \\
\midrule
Optimizer & AdamW \\
Learning rate (LLaMA2-7B) & $3\times10^{-5}$ \\
Learning rate (LLaMA2-13B) & $1\times10^{-5}$ \\
Learning rate (Falcon-40B) & $5\times10^{-6}$ \\
Batch size (per GPU) & 32 \\
Max epochs & 5 \\
Warmup ratio & 0.1 \\
Gradient clipping & 1.0 \\
Precision & FP16 \\
\bottomrule
\end{tabular}
\end{table}

\paragraph{Partition-Based Multi-Stage Fine-Tuning.}
We employ two consecutive stages ($M{=}2$) by default: stage 1 adapts the cluster with higher internal \textit{synergy}, stage 2 covers the remainder.  
Both the \emph{domain discrepancy} \(d(\mathcal{D}_i,\mathcal{D}_j)\) and the \emph{synergy score} \(\mathrm{Syn}(\mathcal{D}_i,\mathcal{D}_j)\) are computed \emph{off-line} from raw text:

Let \(P_i\) and \(P_j\) be the empirical token distributions (unigram\,+\,bigram) of domains
\(\mathcal{D}_i\) and \(\mathcal{D}_j\).
We define  
\begin{equation}
\textstyle 
d_{\textsc{JS}}(\mathcal{D}_i,\mathcal{D}_j)=
\tfrac12\,\mathrm{KL}\!\bigl(P_i\|M\bigr)+
\tfrac12\,\mathrm{KL}\!\bigl(P_j\|M\bigr),
\quad 
M=\tfrac12\!\left(P_i+P_j\right),
\end{equation}
where \(\mathrm{KL}\) is the Kullback-Leibler divergence.  
We normalise \(d_{\textsc{JS}}\in[0,1]\) by dividing by \(\log 2\).  

For \textit{synergy} we linearly blend lexical and semantic overlap:
\begin{equation}
\mathrm{Syn}(\mathcal{D}_i,\mathcal{D}_j)=
\tfrac12\,\Bigl(
\operatorname{Jacc}\bigl(V_i,V_j\bigr)+
\cos\bigl(\mu_i,\mu_j\bigr)
\Bigr),
\end{equation}
where \(V_i\) is the vocabulary set of \(\mathcal{D}_i\),  
\(\operatorname{Jacc}(V_i,V_j)=|V_i\cap V_j|/|V_i\cup V_j|\), and  
\(\mu_i\) is the mean Sentence-BERT embedding of domain \(i\).  
Both terms are in \([0,1]\); the average is therefore in \([0,1]\).

Table~\ref{tab:metric-ablation} contrasts four partition criteria on the 4-domain slice  
(SQuAD, HotpotQA, CNN/DM, XSum).  
Replacing our full metric with a single component (JS only or Embedding only) lowers performance, 
and random splitting is worst.

\begin{table}[h]
\centering
\small
\caption{Impact of different partition metrics (LLaMA2-7B).}
\label{tab:metric-ablation}
\begin{tabular}{lcc}
\toprule
\textbf{Metric for }$\mathcal{G}$ & Q\&A (F1) & NSum (ROUGE-L)\\
\midrule
Random split                    & 68.1 & 39.0 \\
JS divergence only              & 69.3 & 40.0 \\
Embedding cosine only           & 69.6 & 40.3 \\
\textbf{JS + Vocab/Embed (ours)}& \textbf{70.5} & \textbf{40.9} \\
\bottomrule
\end{tabular}
\end{table}

The joint metric gives a further \(+1.2\) F1 / \(+0.9\) ROUGE-L over its best single-signal variant, confirming that \emph{both} lexical statistics and semantic proximity are needed to capture cross-domain relationships effectively.

During each stage, we impose norm constraints $\|\theta^t - \theta^{t-1}\|\le\rho_\theta$ and $\|\phi_j^t\|\le\rho_\phi$ (Assumption~\ref{assump:adapter-capacity}). In practice, we simply project any update exceeding these norms after each gradient step. By default, we set $(\rho_\theta,\rho_\phi)=(0.1,0.1)$ unless specified otherwise.  

\section{Additional experimental results}
\label{app:addi_ex}
\subsection{Effect of Stage Ordering in Multi-Stage Fine-Tuning.}
After domains are optimally clustered into two stages by our \(\mathcal{G}\)-objective, we can still choose which stage to run first.  
To verify that this \emph{ordering} is an implementation detail, we tried three sequences on the same 4-domain slice (SQuAD~\citep{rajpurkar2016squad}, HotpotQA~\citep{yang2018hotpotqa}, CNN/DM\footnote{https://github.com/deepmind/rc-data}, XSum~\citep{narayan2018don}) using LLaMA2-7B:
1) \textbf{High\(\rightarrow\)Low} - the default: high-synergy Q\&A first, summarisation second;
2) \textbf{Low\(\rightarrow\)High} - reverse order;
3) \textbf{Interleaved} - fine-tune one epoch on stage 1, then one epoch on stage 2, repeating until convergence.

\begin{table}[h]
\centering
\small
\caption{Influence of stage ordering (\(M{=}2\)).  Metrics: F1 (Q\&A) / ROUGE-L (NSum).}
\label{tab:order}
\begin{tabular}{lcc}
\toprule
\textbf{Ordering} & \textbf{Q\&A (F1)} & \textbf{NSum (ROUGE-L)}\\
\midrule
High\(\rightarrow\)Low (default) & 70.5 & 40.9 \\
Low\(\rightarrow\)High           & 70.4 & 40.8 \\
Interleaved                      & 70.3 & 40.7 \\
\bottomrule
\end{tabular}
\end{table}

All three runs land within \(0.2\) points of one another (Table~\ref{tab:order}), well inside normal tuning noise, indicating that \textbf{stage ordering has negligible impact}.  
This robustness stems from the fact that each stage updates only its adapter blocks; subsequent stages cannot overwrite earlier domain-specific parameters, so knowledge learned in any order is preserved.

\subsection{Why conventional DA baselines lag behind FULL.}
To investigate why conventional domain adaptation (DA) baselines (MDAN~\citep{pei2018multiadversarial}, M$^3$SDA~\citep{peng2019moment}, GMDI~\citep{ling2024bayesian}) consistently underperform relative to FULL fine-tuning, we performed additional diagnostic analyses. Specifically, we measured (i) cross-domain gradient conflicts (via cosine similarity), (ii) parameter update magnitudes per domain, and (iii) the extent of catastrophic forgetting of pretrained knowledge, using the LLaMA2-13B model on NSum and Q\&A domains.
Table~\ref{tab:additional-analysis} summarizes the results of these additional experiments:

\begin{table}[ht]
\centering
\caption{Diagnostic analyses comparing conventional DA methods against FULL and PMS-FTP.}
\label{tab:additional-analysis}
\resizebox{0.9\linewidth}{!}{%
\begin{tabular}{lccc}
\toprule
\textbf{Method} & \textbf{Gradient Conflict (Cosine Similarity)} & \textbf{Avg. Update Norm} & \textbf{Perplexity Increase (\%)} \\
\midrule
FULL & 0.43 & 2.15 & +5.6 \\
MDAN & 0.12 & 1.48 & +11.5 \\
M$^3$SDA & 0.15 & 1.32 & +9.7 \\
GMDI & 0.18 & 1.27 & +8.3 \\
\midrule
\textbf{PMS-FTP (Ours)} & \textbf{0.57} & \textbf{1.86} & \textbf{+3.2} \\
\bottomrule
\end{tabular}}
\end{table}

Our analyses reveal the following insights:

1) \textbf{Gradient conflicts.} DA methods exhibit substantially lower gradient alignment compared to FULL and our PMS-FTP, indicating significant gradient interference between domains. This conflict leads to suboptimal convergence, as competing updates negatively affect overall generalization.

2) \textbf{Parameter update magnitudes.} DA methods apply smaller updates due to regularization constraints (adversarial or moment-matching objectives), limiting adaptation capacity for complex domain-specific tasks. In contrast, our PMS-FTP method achieves balanced updates via strategic domain partitioning and parameter-efficient adapters.

3) \textbf{Catastrophic forgetting.} Conventional DA methods significantly increase perplexity relative to FULL, indicating stronger forgetting of pretrained representations. Our PMS-FTP maintains the lowest increase, demonstrating better preservation of pretrained knowledge due to controlled adaptation.

In summary, DA methods lag behind FULL due to severe gradient conflicts, overly conservative parameter updates caused by adversarial/matching regularization, and more pronounced forgetting of pretrained knowledge. Our PMS-FTP framework effectively addresses these challenges through synergy-aware partitioning, balanced updates, and controlled adaptation, resulting in superior multi-domain performance.

% =========================
% Appendix: Additional Experiments (full text + uniform tables)
% (Add once in your preamble)
% \usepackage{tabularx, booktabs, array}
% \newcolumntype{Y}{>{\centering\arraybackslash}X}
% =========================

\subsection{Affinity metric alternatives}
\label{app:affinity-metric}
We replace the default affinity used in the partition objective \(\mathcal{G}\) with several variants on the same 4-domain slice (LLaMA2-7B).  
Table~\ref{tab:affinity-metric} shows that our lightweight JS+vocab/embedding signal consistently outperforms single-source metrics or a random split.  
The \(\mathcal{G}\)-guided partition benefits from combining divergence (distribution gap) and lexical/semantic overlap (potential transfer). Using either component alone underestimates complementary effects, yielding weaker partitions and lower task scores.

\begin{table}[ht]
\centering
\caption{Alternative affinity metrics for \(\mathcal{G}\)-guided partitioning (LLaMA2-7B, 4-domain slice).}
\label{tab:affinity-metric}
{\small\setlength{\tabcolsep}{6pt}\renewcommand{\arraystretch}{1.15}
\begin{tabularx}{\linewidth}{lYY}
\toprule
\textbf{Metric for \(\mathcal{G}\)} & \textbf{Q\&A (F1)} & \textbf{NSum (ROUGE-L)}\\
\midrule
Random split                  & 68.1 & 39.0 \\
JS divergence only            & 69.3 & 40.0 \\
Embedding cosine only         & 69.6 & 40.3 \\
\textbf{JS + Vocab/Embed (ours)} & \textbf{70.5} & \textbf{40.9} \\
\bottomrule
\end{tabularx}}
\end{table}

\subsection{Robustness to gradient-based variants and stochastic perturbations}
\label{app:grad-noise-robust}
We (i) sweep \(\lambda\) to stress-test synergy weighting, (ii) add a gradient-similarity component (cosine of per-domain gradients), and (iii) inject Gaussian noise into the heuristic affinities.  
Table~\ref{tab:grad-noise} shows all variants remain within \(\le 0.3\) points of the default in Table~\ref{tab:affinity-metric}.  
he partition is flat around the optimum: gradient-augmented scores add computational cost but negligible gains; moderate \(\lambda\) values preserve the best trade-off between synergy and discrepancy.

\begin{table}[ht]
\centering
\caption{Partition robustness (LLaMA2-7B, 4-domain slice).}
\label{tab:grad-noise}
{\small\setlength{\tabcolsep}{6pt}\renewcommand{\arraystretch}{1.15}
\begin{tabularx}{\linewidth}{lYY}
\toprule
\textbf{Variant} & \textbf{Q\&A (F1)} & \textbf{NSum (ROUGE-L)} \\
\midrule
\(\lambda=0\) (no synergy)                        & 70.3 & 40.8 \\
\(\lambda=1.0\) (synergy-only)                    & 70.2 & 40.6 \\
Gradient-mix (0.7 heuristic + 0.3 \(\nabla\)cos)  & 70.4 & 40.8 \\
Heuristic + Gaussian noise (\(\sigma=0.05\))      & 70.2 & 40.7 \\
\textbf{Default (Table~\ref{tab:affinity-metric})} & \textbf{70.5} & \textbf{40.9} \\
\bottomrule
\end{tabularx}}
\end{table}

\subsection{Scalability to many domains}
\label{app:scalability-k}
We synthetically vary the number of domains \(k\) and measure CPU partition overheads and peak GPU memory with 4-bit LoRA.  
Table~\ref{tab:large-k} indicates sub-second CPU time and practical GPU usage up to \(k{=}50\).  
Partitioning is CPU-side and negligible relative to PEFT training; memory remains dominated by standard SFT/PEFT, confirming practicality at double-digit \(k\).

\begin{table}[ht]
\centering
\caption{Large-\(k\) partition costs and peak GPU memory (synthetic up to \(k{=}50\); A100).}
\label{tab:large-k}
{\small\setlength{\tabcolsep}{6pt}\renewcommand{\arraystretch}{1.15}
\begin{tabularx}{\linewidth}{lYYY}
\toprule
\textbf{\(k\)} & \textbf{Affinity build (time / RAM)} & \textbf{Clustering time} & \textbf{Peak GPU (4-bit LoRA)} \\
\midrule
20 & 0.47 s / 180 MB & 0.14 s & 19 GB \\
35 & 2.10 s / 620 MB & 0.52 s & 23 GB \\
50 & 3.20 s / 950 MB & 0.90 s & 26 GB \\
\bottomrule
\end{tabularx}}
\end{table}

\subsection{Scaling to twelve domains (real mixture)}
\label{app:twelve-domains}
We combine Wiki-10 (topic classification) and Multi-News (summarization), deduplicated to 12 domains, and keep the same hyper-parameters as the main study.  
Table~\ref{tab:12dom} shows gains over all-in-one SFT and over a random 2-stage split while keeping memory low.  
The synergy-aware split generalizes beyond four domains to a heterogeneous, double-digit regime with consistent improvements.

\begin{table}[ht]
\centering
\caption{Twelve-domain mixture (LLaMA2-7B + LoRA).}
\label{tab:12dom}
{\small\setlength{\tabcolsep}{6pt}\renewcommand{\arraystretch}{1.15}
\begin{tabularx}{\linewidth}{lYYYY}
\toprule
\textbf{Split strategy} & \textbf{Avg.\ ACC (Wiki-10)} & \textbf{ROUGE-L (Multi-News)} & \textbf{Peak GPU (GB)} & \textbf{Partition time (s)} \\
\midrule
All-in-one SFT    & 83.1 & 37.2 & 27.3 & n/a \\
Random 2-stage    & 83.7 & 37.5 & 18.6 & 0.6 \\
\textbf{PMS-FTP (ours)} & \textbf{84.0} & \textbf{38.1} & 18.7 & 0.7 \\
\bottomrule
\end{tabularx}}
\end{table}

\subsection{Inference footprint after LoRA merging}
\label{app:merge-memory}
After each stage we merge the finished LoRA into the frozen backbone, so only one 4-bit adapter is carried at inference.  
Table~\ref{tab:merge-memory} shows equal-or-lower memory than a single-adapter baseline.  
Together with the accuracy gains in Table~\ref{tab:results-main}, merging achieves a strictly better accuracy-memory trade-off than training/keeping multiple adapters.

\begin{table}[ht]
\centering
\caption{Measured inference memory (LLaMA2-7B, \(k{=}4\), A100).}
\label{tab:merge-memory}
{\small\setlength{\tabcolsep}{6pt}\renewcommand{\arraystretch}{1.15}
\begin{tabularx}{\linewidth}{lYYY}
\toprule
\textbf{Precision} & \textbf{Tag-LLM (1 LoRA)} & \textbf{PMS-FTP (merged)} & \(\boldsymbol{\Delta}\) \\
\midrule
FP16 & 29.4 GB & 28.7 GB & \(-2.4\%\) \\
INT8 & 19.1 GB & 18.6 GB & \(-2.6\%\) \\
4-bit & 17.2 GB & 16.8 GB & \(-0.4\) GB \\
\bottomrule
\end{tabularx}}
\end{table}

\subsection{Empirical validity of the \texorpdfstring{\(\mathcal{G}\)}{G} objective}
\label{app:g-corr}
We sample 20 random partitions, compute \(\mathcal{G}\), and measure worst-domain dev loss.  
Table~\ref{tab:g-corr} reports Pearson \(\rho=-0.81\) (\(p<0.01\)), i.e., higher \(\mathcal{G}\) predicts lower worst-domain error.  
This supports the practical usefulness of our bound-driven objective: \(\mathcal{G}\) values correlate strongly with the metric we aim to improve.

\begin{table}[ht]
\centering
\caption{Correlation between \(\mathcal{G}\) and worst-domain dev loss (LLaMA2-7B, 20 random partitions).}
\label{tab:g-corr}
{\small\setlength{\tabcolsep}{6pt}\renewcommand{\arraystretch}{1.15}
\begin{tabularx}{\linewidth}{lYY}
\toprule
\textbf{Statistic} & \(\boldsymbol{\mathcal{G}}\) & \textbf{Worst-Dev Loss} \\
\midrule
Mean & 0.432 & 1.72 \\
Std  & 0.057 & 0.19 \\
Min  & 0.318 & 1.38 \\
Max  & 0.522 & 2.11 \\
\midrule
\multicolumn{3}{c}{\textbf{Pearson } \(\boldsymbol{\rho=-0.81}\) \(\;(p<0.01,\;R^2\!\approx\!0.65)\)} \\
\bottomrule
\end{tabularx}}
\end{table}

\subsection{Reasoning benchmarks and reweighting baselines}
\label{app:reasoning-reweight}
We extend evaluation to HellaSwag, MMLU-STEM, ARC-easy, SciQ, and GSM8K, and add reweighting-based MTL baselines (iMTL, FAMO, ExcessMTL) under identical PEFT budgets.  
Tables~\ref{tab:reason-7b}--\ref{tab:reason-13b} show PMS-FTP achieves the highest average per backbone.  
Synergy-aware partitioning is not limited to \(\{\)NSum, Sent, Q\&A, Topic\(\}\): it transfers to reasoning-heavy suites and remains competitive against strong MTL optimizers.

\begin{table}[ht]
\centering
\caption{Reasoning tasks with 7B backbone (identical PEFT budgets).}
\label{tab:reason-7b}
{\small\setlength{\tabcolsep}{6pt}\renewcommand{\arraystretch}{1.15}
\begin{tabularx}{\linewidth}{lYYYYYY}
\toprule
\textbf{Method} & \textbf{Hella\-swag (Acc)} & \textbf{MMLU-STEM (Acc)} & \textbf{ARC-easy (Acc)} & \textbf{SciQ (Acc)} & \textbf{GSM8K (Pass@1)} & \textbf{Avg.} \\
\midrule
FULL        & 73.4 & 39.7 & 78.5 & 92.6 & 18.1 & 60.5 \\
LoRA        & 73.1 & 39.3 & 77.9 & 92.4 & 17.5 & 60.0 \\
Tag-LLM     & 74.8 & 41.1 & 79.6 & 93.3 & 19.3 & 61.6 \\
iMTL        & 74.6 & 40.6 & 79.3 & 93.0 & 18.6 & 61.2 \\
FAMO        & 73.7 & 41.3 & 78.6 & 92.6 & 18.7 & 60.8 \\
ExcessMTL   & 74.2 & 40.5 & 79.7 & 92.1 & 17.5 & 60.8 \\
\textbf{PMS-FTP (ours)} & \textbf{75.6} & \textbf{42.1} & \textbf{80.6} & \textbf{93.8} & \textbf{19.9} & \textbf{62.4} \\
\bottomrule
\end{tabularx}}
\end{table}

\begin{table}[ht]
\centering
\caption{Reasoning tasks with 13B backbone (identical PEFT budgets).}
\label{tab:reason-13b}
{\small\setlength{\tabcolsep}{6pt}\renewcommand{\arraystretch}{1.15}
\begin{tabularx}{\linewidth}{lYYYYYY}
\toprule
\textbf{Method} & \textbf{Hella\-swag (Acc)} & \textbf{MMLU-STEM (Acc)} & \textbf{ARC-easy (Acc)} & \textbf{SciQ (Acc)} & \textbf{GSM8K (Pass@1)} & \textbf{Avg.} \\
\midrule
FULL        & 77.5 & 46.3 & 83.7 & 94.8 & 24.5 & 65.4 \\
LoRA        & 77.1 & 45.9 & 83.2 & 94.6 & 23.8 & 64.9 \\
Tag-LLM     & 79.0 & 47.3 & 84.7 & 95.4 & 25.1 & 66.3 \\
iMTL        & 77.4 & 47.2 & 83.3 & 95.1 & 24.8 & 65.6 \\
FAMO        & 77.7 & 47.6 & 84.5 & 95.2 & 24.2 & 65.8 \\
ExcessMTL   & 78.1 & 47.0 & 83.1 & 95.0 & 23.8 & 65.4 \\
\textbf{PMS-FTP (ours)} & \textbf{79.9} & \textbf{49.0} & \textbf{85.7} & \textbf{95.8} & \textbf{26.8} & \textbf{67.4} \\
\bottomrule
\end{tabularx}}
\end{table}

\subsection{Incremental addition of a new domain}
\label{app:incremental}
After training on the original four domains, we add Legal-QA as a new stage and freeze prior adapters.  
Table~\ref{tab:incremental} shows negligible forgetting on old domains and a gain over single-domain LoRA on the new domain.  
Disjoint, frozen adapters make PMS-FTP naturally amenable to one-shot domain extension without replay.

\begin{table}[ht]
\centering
\caption{One-shot incremental stage (add Legal-QA).}
\label{tab:incremental}
{\small\setlength{\tabcolsep}{6pt}\renewcommand{\arraystretch}{1.15}
\begin{tabularx}{\linewidth}{lYY}
\toprule
\textbf{Model} & \textbf{Avg.\ score on original 4 domains} & \textbf{Legal-QA (EM)} \\
\midrule
Before add-on           & 65.5 & -- \\
After add-on (PMS-FTP)  & 65.4 & 68.2 \\
Single-domain LoRA      & n/a  & 67.6 \\
\bottomrule
\end{tabularx}}
\end{table}

\subsection{Stability of domain centroids}
\label{app:centroid-stability}
We (i) bootstrap mean SBERT embeddings to gauge centroid noise, and (ii) replace each single centroid by a \(K{=}3\) weighted barycenter.  
Table~\ref{tab:bootstrap-centroid} shows bootstrap deviations are \(<3\%\) of the smallest inter-domain distance; Table~\ref{tab:multi-centroid} shows \(K{=}3\) centroids change final metrics by \(\le 0.1\) pp.  
A single mean embedding is a sufficiently stable domain signature for \(\mathcal{G}\)-guided clustering.

\begin{table}[ht]
\centering
\caption{Bootstrap deviation of domain mean embeddings and cross-domain distances.}
\label{tab:bootstrap-centroid}
{\small\setlength{\tabcolsep}{6pt}\renewcommand{\arraystretch}{1.15}
\begin{tabularx}{\linewidth}{lYYY}
\toprule
\textbf{Domain} & \textbf{\(N\)} & \textbf{95\% CI of \(\|\Delta\mu\|_2\)} & \textbf{Cross-domain min dist} \\
\midrule
NSum  & 12{,}000 & [0.038, 0.065] & 1.92 \\
Q\&A  & 15{,}000 & [0.031, 0.060] & 2.04 \\
Sent  & 10{,}500 & [0.042, 0.072] & 1.87 \\
Topic & 11{,}200 & [0.040, 0.068] & 1.99 \\
\bottomrule
\end{tabularx}}
\end{table}

\begin{table}[ht]
\centering
\caption{Single centroid vs.\ \(K{=}3\) weighted barycenters per domain.}
\label{tab:multi-centroid}
{\small\setlength{\tabcolsep}{6pt}\renewcommand{\arraystretch}{1.15}
\begin{tabularx}{\linewidth}{lYYYYY}
\toprule
\textbf{Representation} & \textbf{ROUGE-L (NSum)} & \textbf{EM (Q\&A)} & \textbf{Sent (ACC)} & \textbf{Topic (ACC)} & \(\boldsymbol{\Delta}\) vs.\ 1-centroid \\
\midrule
1 centroid (paper)     & 43.4 & 67.2 & 90.2 & 88.0 & -- \\
3 centroids (K=3)      & 43.3 & 67.1 & 90.1 & 87.9 & \(-0.1\) pp \\
\bottomrule
\end{tabularx}}
\end{table}

\section{Proofs}
\subsection{Complete Proof for Theorem \ref{thm:multi-domain-gen}}
\label{appendix:thm-concurrent-full-proof}
Before formally beginning the proof, we first revisit the setting and present a key lemma:

There are $k$ source domains $\{\mathcal{D}_j\}_{j=1}^k$, each domain $\mathcal{D}_j$ with $n_j$ samples, total $n=\sum_{j=1}^k n_j$. We let $\alpha_j=\tfrac{n_j}{n}$ or any other nonnegative weighting such that $\sum_{j=1}^k \alpha_j=1$. Our LLM-based predictor $f_{\theta,\{\phi_j\}}$ is constrained so that 
\begin{equation}
\|\theta-\theta^*\|_2 \;\le\;\rho_\theta,
\quad
\|\phi_j\|_2 \;\le\;\rho_\phi,
\end{equation}
for each $j$. By Assumption~\ref{assump:lipschitz}, the model output is $L$-Lipschitz w.r.t.\ predictions, and the difference in outputs for different parameters is bounded by $B(\|\theta-\theta'\| + \sum_j\|\phi_j-\phi_j'\|)$. Thus the entire class of such $(\theta,\{\phi_j\})$ belongs to a low-capacity function family $\mathcal{F}$.

\begin{lemma}
\label{lem:rademacher}
Let $\mathcal{F}$ be the class of predictors with backbone/adapters bounded as above.
Then for any $\delta\in(0,1)$, with probability at least $1-\delta$ over all
$n$ samples,
\begin{equation}
\bigl|
\mathcal{L}^{\boldsymbol{\alpha}}(f)-\widehat{\mathcal{L}}^{\boldsymbol{\alpha}}(f)
\bigr|
\;\;\le\;\;
2LB\bigl(\rho_\theta+\sum_{j=1}^{k}\alpha_j\rho_\phi\bigr)
\;+\;
\sqrt{\tfrac{\ln(2/\delta)}{2n}}
\qquad
\forall\,f\in\mathcal{F},
\end{equation}
where $\widehat{\mathcal{L}}^{\boldsymbol{\alpha}}(f)=\sum_j\alpha_j\widehat{\mathcal{L}}_{\mathcal{D}_j}(f)$.
\end{lemma}

\begin{proof}
Because the loss is $L$-Lipschitz and the model satisfies the output-difference
bound from Assumption~\ref{assump:lipschitz},
each $f\in\mathcal{F}$ is $LB(\rho_\theta+\sum_j\rho_\phi)$-Lipschitz in its
parameters relative to $\ell_2$.
Let $\mathcal{R}_n(\mathcal{F})$ be the empirical Rademacher complexity over the
pooled sample of size~$n$.  Standard contraction (e.g.\ \citealp{mohri2018foundations}) yields
\begin{equation}
\mathcal{R}_n(\mathcal{F})
\;\le\;
LB\,(\rho_\theta+k\rho_\phi)\,\frac{1}{\sqrt{n}}.
\end{equation}
Replacing $k\rho_\phi$ by $\sum_j\alpha_j\rho_\phi$ (because losses are
weighted by~$\alpha_j$) strengthens the constant.  Applying the usual
Rademacher tail bound with a union-bound over $\delta/2$ produces the stated
inequality.
\end{proof}

\begin{proof}
Here is the proof of Theorem \ref{thm:multi-domain-gen}.
\paragraph{Uniform convergence (empirical $\rightarrow$ true risk)}
\label{proofstep:UC}
Define
\(
\Gamma\bigl(\rho_\theta,\rho_\phi,\boldsymbol{\alpha},k\bigr)
:=2LB\bigl(\rho_\theta+\sum_j\alpha_j\rho_\phi\bigr).
\)
Lemma~\ref{lem:rademacher} gives
\begin{equation}
\label{eq:UC-final}
\mathcal{L}^{\boldsymbol{\alpha}}(f)
\;\le\;
\widehat{\mathcal{L}}^{\boldsymbol{\alpha}}(f)
\;+\;
\Gamma(\rho_\theta,\rho_\phi,\boldsymbol{\alpha},k)
\;+\;
\sqrt{\tfrac{\ln(2/\delta)}{2n}}.
\end{equation}

%-----------------------------------------------------------------------
\paragraph{Domain-discrepancy correction}

Blitzer\,\emph{et al.}\ \citep{blitzer2008domain} show that for any hypothesis
$h$ and any two distributions $\mathcal{P},\mathcal{Q}$,
\(
|\mathcal{L}_{\mathcal{P}}(h)-\mathcal{L}_{\mathcal{Q}}(h)|
\le
d_{\mathcal{H}\!\Delta\!\mathcal{H}}(\mathcal{P},\mathcal{Q}).
\)
Summing over all pairs and using triangle inequality,
\begin{equation}
\label{eq:disc}
\mathcal{L}^{\boldsymbol{\alpha}}(f)
\;\le\;
\sum_{j=1}^{k}\alpha_j\mathcal{L}_{\mathcal{D}_j^{\!\!*}}(f)
\;+\;
\frac{1}{k}\sum_{i,j=1}^{k}d(\mathcal{D}_i,\mathcal{D}_j),
\end{equation}
where $\mathcal{D}_j^{\!\!*}$ denotes drawing \emph{as if} every example came
from a single mixture domain-hence its empirical risk is exactly
$\widehat{\mathcal{L}}^{\boldsymbol{\alpha}}(f)$, and the additive discrepancy
penalty is weighted by $\beta:=1$ (absorbing the Lipschitz loss factor into the
definition of $d$).  Restoring the constant gives the $\beta$ appearing in the
theorem.

%-----------------------------------------------------------------------
\paragraph{Combine bounds}
Insert \eqref{eq:UC-final} into \eqref{eq:disc}:
\[
\mathcal{L}^{\boldsymbol{\alpha}}(f)
\;\le\;
\widehat{\mathcal{L}}^{\boldsymbol{\alpha}}(f)
\;+\;
\Gamma(\rho_\theta,\rho_\phi,\boldsymbol{\alpha},k)
\;+\;
\frac{\beta}{k}\sum_{i,j=1}^{k}d(\mathcal{D}_i,\mathcal{D}_j)
\;+\;
\sqrt{\tfrac{\ln(2/\delta)}{2n}}.
\]
Replacing $\sqrt{\ln(2/\delta)/(2n)}$ by the big-$O(\sqrt{\ln(1/\delta)/n})$
notation and recalling $\widehat{\mathcal{L}}^{\boldsymbol{\alpha}}(f)
=\sum_j\alpha_j\widehat{\mathcal{L}}_{\mathcal{D}_j}(f)$
yields exactly inequality~\eqref{eq:multi-domain-bound-append}, proving
Theorem~\ref{thm:multi-domain-gen}.
\end{proof}

\subsection{Complete Proof for Theorem \ref{thm:advanced-partition}}
~\label{app:partition-proof}
Here is the proof of Theorem \ref{thm:advanced-partition}.
\begin{proof}
For any fixed stage $t$ training on domains $S_t$,  
Theorem~\ref{thm:multi-domain-gen} with weights
$\alpha_{j}^{t}:=\tfrac{n_j}{\sum_{i\in S_t}n_i}$ gives
\begin{align}
\sum_{j\in S_t}\alpha_{j}^{t}\,
      \mathcal{L}_{\mathcal{D}_j}(f^{t})
&\le
\underbrace{\sum_{j\in S_t}\alpha_{j}^{t}\,
      \widehat{\mathcal{L}}_{\mathcal{D}_j}(f^{t})}_{\le1}
+\;2LB\!\bigl(\rho_\theta+\rho_\phi\bigr)
+\;\beta\,\!\!\!\sum_{\substack{i,j\in S_t\\ i<j}}
      d(\mathcal{D}_{i},\mathcal{D}_{j})
+\;O\!\bigl(\sqrt{\tfrac{\ln(1/\delta)}{\sum_{j\in S_t}n_j}}\bigr).
\label{eq:stage-bound}
\end{align}

Then we inject synergy and explicit capacity weight.
Define\;
$\mathrm{Cap}(S_t):=\mu_\theta\|\Delta\theta^{t}\|_{2}^{2}
                      +\mu_\phi\sum_{j\in S_t}\|\phi^{t}_{j}\|_{2}^{2}$.
Because $\|\Delta\theta^{t}\|\!\le\!\rho_\theta$ and
$\|\phi_{j}^{t}\|\!\le\!\rho_\phi$, we upper bound
$2LB(\rho_\theta+\rho_\phi)$ by~$\mathrm{Cap}(S_t)$ after tuning
$\mu_\theta,\mu_\phi$.  
Subtract and add
$\lambda\!\sum_{i<j}s(\mathcal{D}_{i},\mathcal{D}_{j})$
to~\eqref{eq:stage-bound} to obtain
\begin{equation}
\label{eq:stage-risk-vs-G}
\sum_{j\in S_t}\alpha_{j}^{t}\,\mathcal{L}_{\mathcal{D}_j}(f^{t})
\;\le\;
1
-\Bigl[
      -\!\!\!\sum_{i<j\in S_t}\!\!
       \bigl(d-\lambda s\bigr)
      -\mathrm{Cap}(S_t)
 \Bigr]
\;+\;O\!\bigl(\sqrt{\tfrac{\ln(1/\delta)}{N}}\bigr).
\end{equation}

Let
$
\mathcal{R}_{\max}(S_1,\ldots,S_M)
:=\max_{t}{\textstyle\sum_{j\in S_t}\alpha_{j}^{t}\,
           \mathcal{L}_{\mathcal{D}_j}(f^{t})}.
$
Taking the maximum of \eqref{eq:stage-risk-vs-G} over $t$ gives
\begin{equation}
\mathcal{R}_{\max}(S_1,\ldots,S_M)
\;\le\;
1-\mathcal{G}(S_1,\ldots,S_M)
\;+\;O\!\bigl(\sqrt{\tfrac{\ln(1/\delta)}{N}}\bigr),
\end{equation}
because the bracketed term is exactly the $t$-th summand of
$\,\mathcal{G}$ in~\eqref{eq:G-def}.
Define
\(
\mathcal{B}(u):=[1-u]_{+}.
\)
Then
\(
\mathcal{R}_{\max}\!\bigl(S_1,\ldots,S_M\bigr)
 \le\mathcal{B}(\mathcal{G}(S_1,\ldots,S_M))
      +O(\sqrt{{\ln(1/\delta)}/{N}}).
\)

Because $\mathcal{B}$ is \emph{strictly decreasing} on $(-\infty,1]$,
maximising $\mathcal{G}$ minimises the bound.
Hence the partition $(S_1^{*},\ldots,S_M^{*})$---the maximiser of
$\mathcal{G}$---realises the smallest upper-bound, yielding
\eqref{eq:multi-stage-bound}.  Any other split attains a weaker bound,
completing the proof.
\end{proof}

\end{document}